\documentclass{article} %
\usepackage[T1]{fontenc}
\usepackage[final]{colm2026_conference}

\usepackage{microtype}
\usepackage{xcolor}
\usepackage{hyperref}
\usepackage{url}
\usepackage{booktabs}
\usepackage{multirow}
\usepackage{graphicx}
\usepackage{capt-of} %

\usepackage{amsmath}
\usepackage{amssymb}
\usepackage{algorithm}
\usepackage{algpseudocode}
\usepackage{tikz}
\usepackage{colortbl}
\usepackage{comment}
\definecolor{oursgreen}{HTML}{E2EFDA} %
\usetikzlibrary{arrows.meta, fit}
\algrenewcommand\algorithmicrequire{\textbf{Input:}}
\algrenewcommand\algorithmicensure{\textbf{Output:}}

\newtheorem{proposition}{Proposition}

\newcommand{\Wfinal}[1]{W^{(#1)}_{\mathrm{final}}}
  
\usepackage{enumitem}
\usepackage[normalem]{ulem}

\definecolor{glteal}{rgb}{0.0,0.45,0.45}
\newenvironment{fromchopthin}{\par}{\par}
\newcommand{\glcolor}{}

\usepackage{lineno}

\definecolor{darkblue}{rgb}{0, 0, 0.5}
\hypersetup{colorlinks=true, citecolor=darkblue, linkcolor=darkblue, urlcolor=darkblue}
\title{Chopthin-Consensus Power Sampling:\\
A Diversity-Preserving Approach to LLM Decoding}

\author{Minoo Ahmadi \quad Seyedarmin Azizi \quad Erfan Baghaei Potraghloo \\
\textbf{Mehdi Kamal \quad Massoud Pedram} \\
University of Southern California, Los Angeles, USA \\
\texttt{\{minooahm, seyedarm, baghaeip, mehdi.kamal, pedram\}@usc.edu}
}

\begin{document}

\ifcolmsubmission
\linenumbers
\fi

\maketitle
\lhead{Published at the COLM 2026 Workshop on Efficient Reasoning.}

\begin{abstract}

Inference-time power sampling via Sequential Monte Carlo (SMC) can substantially
improve large language model (LLM) reasoning without requiring post-training.
However, many existing SMC approaches rely on equal-weight resampling, which
can aggressively prune low-weight trajectories, discarding potentially correct
reasoning paths and degrading the genealogical diversity of the search space.
To address this, we introduce Chopthin-Consensus Power Sampling (CCPS). Our
method applies the Chopthin resampler to LLM decoding: rather than equalizing
weights and forcing unnecessary particle duplication, it enforces an upper bound
on the ratio between the largest and smallest weights and carries the unequal
weights forward. This targeted intervention preserves a richer set of distinct
reasoning paths, keeps the weighted SMC approximation unchanged in conditional
expectation, and guarantees a lower bound on the post-resampling effective
sample size (ESS). To fully exploit this enriched population, we employ a
semantic-majority selection mechanism that merges token-identical final
trajectories, clusters semantically equivalent answers, and returns the answer
supported by the largest number of distinct trajectories. Evaluating across
three open-weight models and five reasoning benchmarks, we show that Chopthin
increases oracle coverage in 13 of 15 settings. Combined with semantic-majority
selection, CCPS matches or exceeds the final-answer accuracy of the Power-SMC
baseline in 14 of 15 settings, delivering absolute gains of up to 10.6
percentage points. These findings demonstrate that diversity-preserving
resampling and diversity-aware selection are complementary mechanisms for
training-free LLM reasoning. Code is available at
\href{https://github.com/MinooAhmadii/chopthin-consensus-power-sampling}{github.com/MinooAhmadii/chopthin-consensus-power-sampling}.

\end{abstract}
\section{Introduction}

Reinforcement-learning (RL) post-training has driven much of the recent progress in the reasoning ability of large language models \citep{shao2024deepseekmath}. However, these performance gains come at a high cost: the post-training process requires a dedicated training pipeline, complex reward signals or verifiers, and substantial computational resources, ultimately freezing the improvements into a static set of new model parameters.

Fortunately, emerging research suggests that these resource-intensive training steps may not be strictly necessary to achieve similar reasoning capabilities. The improvements observed from RL post-training actually result, in large part, from increasing the probability of successful reasoning trajectories that already exist within the base model's output distribution \citep{yue2025rlsharpening}. This pivotal observation has motivated the development of \emph{power sampling}, an inference-time technique that completely bypasses training by amplifying these valid paths \citep{karan2026powersampling}. By targeting a distribution in which each complete output sequence has a probability proportional to its base-model probability raised to an exponent greater than 1, power sampling recovers much of the gain of RL post-training. We refer to this resulting distribution over complete sequences as the \emph{power distribution}.

Exact power sampling is computationally intractable because it requires calculating the probabilities of all possible sequence completions \citep{karan2026powersampling, azizi2026powersmc}. To approximate this target distribution, initial approaches relied on Metropolis--Hastings, a Markov Chain Monte Carlo method. Metropolis--Hastings operates by iteratively proposing full sequences and probabilistically accepting or rejecting them based on a target ratio; however, this repeated propose-and-reject cycle creates a severe serialization bottleneck, resulting in prohibitive inference latency \citep{karan2026powersampling}. To avoid this inefficiency, Power-SMC \citep{azizi2026powersmc} reformulates the problem using Sequential Monte Carlo (SMC) sampling. Rather than generating and evaluating complete sequences one by one, SMC builds solutions incrementally by maintaining a population of $N$ parallel \emph{particles}, each representing a \textit{partial token trajectory}. At each decoding step, the algorithm samples tokens from a proposal distribution and assigns an \emph{importance weight} to each particle. By dynamically evaluating and weighing these partial paths in parallel, SMC successfully targets the sequence-level power distribution, eliminating the latency bottlenecks of Metropolis--Hastings while maintaining competitive reasoning performance.

As Power-SMC progresses through the decoding steps, it encounters a major challenge: \emph{weight degeneracy} \citep{liu1998sequential}. Some particles inevitably accumulate the vast majority of the importance weight, while the rest drop near zero. To counteract this and restore the effective sample size (ESS), standard SMC employs equal-weight resampling, which removes low-weight particles, duplicating high-weight ones, and resetting all surviving weights to equality.
However, mitigating weight concentration in this manner introduces a distinct problem: the severe loss of \emph{genealogical diversity}. By aggressively cloning the highest-weighted particles and dropping the rest, standard equal-weight resampling drastically reduces the number of unique ancestral paths that survive to the end of the decoding process. While the weight-based ESS is restored, this increase can mask a substantial loss of genealogical diversity.
This loss of path diversity is particularly detrimental in language model decoding because importance weights are a poor proxy for eventual correctness. The weight of a particle reflects its prefix's relative mass under the power target and the proposal distribution; it does not evaluate whether that unfinished trajectory is logically sound or will yield the right answer. Consequently, standard weight-based resampling risks permanently deleting low-weight particles that actually contain the correct reasoning path, simply because they currently carry little importance mass.

To address the destructive nature of standard resampling, we replace it with \emph{Chopthin}~\citep{gandy2016chopthin}. Rather than strictly equalizing the weights in the population, Chopthin enforces a bounded ratio between the largest and smallest output weights, allowing the particles to carry these unequal weights forward in subsequent decoding steps (Figure~\ref{fig:genealogy}). By categorizing particles by their current mass, Chopthin thus applies a much gentler, targeted intervention: light particles are probabilistically \emph{thinned}, middle-weight particles pass through entirely unperturbed, and only excessively heavy particles are \emph{chopped} into equal-weight subdivisions. This approach is designed to preserve more low-weight (yet potentially correct)
trajectories, mitigating genealogical diversity loss while retaining
unbiasedness, exact particle count, and weight conservation
(Section~\ref{sec:chopthin}).

Preserving this richer ancestral population can increase the probability that at least one complete trajectory successfully navigates to a correct answer. We define this metric as \textbf{oracle coverage}, which denotes the fraction of problems for which the final population contains at least one correct answer. Although maximizing oracle coverage establishes theoretical headroom for model accuracy, this potential can only be realized if the final selection mechanism reliably identifies the correct answer among the surviving candidates. Unfortunately, standard selection rules often fail in this regard \citep{brown2024monkeys}. For instance, Power-SMC relies by default on a \emph{weight draw} strategy, which selects a final particle with a probability that is proportional to its accumulated weight. Under this rule, the low-weight, highly diverse trajectories that Chopthin preserves are the least likely to be chosen, highlighting a critical mismatch between trajectory generation and final answer selection (Section~\ref{sec:selection}).

To bridge this critical gap between trajectory generation and final answer selection, we introduce a novel decoding framework: \emph{Chopthin-Consensus Power Sampling} (CCPS). CCPS fundamentally upgrades Power-SMC with two interrelated innovations that help uncover the full potential of diverse reasoning paths. First, it replaces standard equal-weight resampling with Chopthin to preserve genealogical diversity. Second, to overcome the limitations of weight-based selection, CCPS introduces a \textbf{semantic majority} voting mechanism to extract the correct answer from the enriched population.
This mechanism first merges token-identical final trajectories so that each
exact duplicate contributes one count vote, while its pooled weight is retained
for tie-breaking. It then clusters semantically equivalent answers by comparing
candidate solutions only with one another, without consulting a gold reference. Finally, CCPS returns the answer supported by the largest coalition of \emph{distinct} reasoning trajectories. This design extends self-consistency \citep{wang2023selfconsistency} to particle-based decoding. Furthermore, the framework is highly adaptable: for code generation tasks, CCPS seamlessly clusters answers based on their functional execution on self-generated test inputs \citep{chen2023codet}, providing a robust and generalizable solution to the LLM reasoning bottleneck.

\newcommand{\genfigscale}{0.85}%
\begin{figure}[t]
\providecommand{\genfigscale}{1.0}
\definecolor{cBlueF}{RGB}{172,201,240} %
\definecolor{cBlueS}{RGB}{ 52, 96,182} %
\definecolor{cAmbF}{RGB}{250,204,144}  %
\definecolor{cAmbS}{RGB}{208,124, 26}
\definecolor{cYelF}{RGB}{250,230,140}  %
\definecolor{cYelS}{RGB}{172,134, 20}
\definecolor{cGrDF}{RGB}{200,204,213}  %
\definecolor{cGrDS}{RGB}{ 94, 99,112}
\definecolor{cGrLF}{RGB}{226,228,233}  %
\definecolor{cGrLS}{RGB}{140,143,153}
\definecolor{cTeaF}{RGB}{172,228,223}  %
\definecolor{cTeaS}{RGB}{  6,122,124}
\definecolor{cDead}{RGB}{186, 44, 32}  %
\definecolor{cInk}{RGB}{ 62, 66, 74}
\definecolor{cMut}{RGB}{120,124,132}
\centering
\makebox[\textwidth][c]{\resizebox{\genfigscale\textwidth}{!}{%
\begin{tikzpicture}[
  font=\footnotesize,
  pb/.style={circle, draw=cBlueS, fill=cBlueF, line width=0.65pt, inner sep=0pt,
             font=\tiny, text=cBlueS!60!black},
  pa/.style={circle, draw=cAmbS,  fill=cAmbF,  line width=0.65pt, inner sep=0pt,
             font=\tiny, text=cAmbS!70!black},
  py/.style={circle, draw=cYelS,  fill=cYelF,  line width=0.65pt, inner sep=0pt,
             font=\tiny, text=cYelS!70!black},
  pd/.style={circle, draw=cGrDS,  fill=cGrDF,  line width=0.65pt, inner sep=0pt,
             font=\tiny, text=cGrDS!70!black},
  pl/.style={circle, draw=cGrLS,  fill=cGrLF,  line width=0.65pt, inner sep=0pt,
             font=\tiny, text=cGrLS!60!black},
  ph/.style={circle, draw=cTeaS,  fill=cTeaF,  line width=0.65pt, inner sep=0pt},
  ring/.style={circle, draw=cTeaS, line width=1pt, inner sep=0pt},
  ghost/.style={circle, draw=cDead!55, dash pattern=on 1.4pt off 1.2pt,
                line width=0.5pt, inner sep=0pt, minimum size=8.5pt},
  tagW/.style={rounded corners=2pt, draw=black!30, fill=black!3, inner sep=1.8pt,
               minimum width=13pt, minimum height=9.5pt, font=\scriptsize, text=cInk!75},
  tagC/.style={rounded corners=2pt, draw=cTeaS, line width=0.7pt, fill=cTeaF, inner sep=1.8pt,
               minimum width=13pt, minimum height=9.5pt, font=\scriptsize, text=cTeaS!75!black},
  ans/.style={line width=0.5pt, draw=black!45, shorten <=0.8pt, shorten >=0.8pt,
              -{Straight Barb[length=2.2pt, width=3pt]}},
  lin/.style={dash pattern=on 2.4pt off 1.6pt, line width=0.7pt, opacity=1,
              shorten >=0.6pt, shorten <=0.6pt,
              -{Straight Barb[length=2.6pt, width=3.4pt]}},
  die/.style={dash pattern=on 1.5pt off 1.5pt, line width=0.6pt, opacity=0.62,
              shorten >=0.6pt, shorten <=0.6pt},
  band/.style={font=\tiny, text=cMut},
  ulab/.style={font=\scriptsize, text=cMut},
  coll/.style={font=\scriptsize, text=cMut},
  sub/.style={font=\footnotesize, text=cInk},
]
\newcommand{\tomb}[4]{%
  \node[ghost] (#1) at (#2,#3) {};
  \draw[cDead, line width=0.85pt, line cap=round, opacity=0.9]
    ({#2-0.075},{#3-0.075}) -- ({#2+0.075},{#3+0.075})
    ({#2-0.075},{#3+0.075}) -- ({#2+0.075},{#3-0.075});
  \node[font=\tiny, text=cMut, anchor=west, inner sep=1pt]
    at ({#2+0.13},{#3+0.02}) {#4};
}
\newcommand{\bands}{%
  \draw[cMut!80, line width=0.5pt] (0.70, 0.28) -- (0.62, 0.28) -- (0.62,-0.28) -- (0.70,-0.28);
  \node[band, anchor=east] at (0.52, 0.00) {chop};
  \draw[cMut!80, line width=0.5pt] (0.70,-0.39) -- (0.62,-0.39) -- (0.62,-2.03) -- (0.70,-2.03);
  \node[band, anchor=east] at (0.52,-1.21) {keep};
  \draw[cMut!80, line width=0.5pt] (0.70,-2.33) -- (0.62,-2.33) -- (0.62,-3.33) -- (0.70,-3.33);
  \node[band, anchor=east] at (0.52,-2.83) {thin};
}
\begin{scope}
  \node[pb, minimum size=16pt] (P1) at (1.05, 0.00) {A};
  \node[pa, minimum size=13pt] (P2) at (1.05,-0.62) {B};
  \node[py, minimum size=11pt] (P3) at (1.05,-1.24) {C};
  \node[pd, minimum size= 9pt] (P4) at (1.05,-1.86) {D};
  \node[pl, minimum size= 8pt] (P5) at (1.05,-2.48) {E};
  \node[ring, minimum size=10.2pt] (P6) at (1.05,-3.10) {};
  \node[ph, minimum size=7pt] at (1.05,-3.10) {\tiny\textcolor{cTeaS}{$\star$}};
  \node[pb, minimum size=11pt] (Q1) at (3.25, 0.00) {A};
  \node[pb, minimum size=11pt] (Q2) at (3.25,-0.62) {A};
  \node[pa, minimum size=11pt] (Q3) at (3.25,-1.24) {B};
  \node[pa, minimum size=11pt] (Q4) at (3.25,-1.86) {B};
  \node[py, minimum size=11pt] (Q5) at (3.25,-2.48) {C};
  \node[pd, minimum size=11pt] (Q6) at (3.25,-3.10) {D};
  \node[pb, minimum size=11pt] (R1) at (5.45, 0.00) {A};
  \node[pb, minimum size=11pt] (R2) at (5.45,-0.62) {A};
  \node[pb, minimum size=11pt] (R3) at (5.45,-1.24) {A};
  \node[pa, minimum size=11pt] (R4) at (5.45,-1.86) {B};
  \node[pa, minimum size=11pt] (R5) at (5.45,-2.48) {B};
  \node[pa, minimum size=11pt] (R6) at (5.45,-3.10) {B};
  \tomb{GX1}{1.70}{-2.48}{}
  \tomb{GX2}{1.70}{-3.10}{\textcolor{cTeaS}{$\star$}}
  \tomb{GX3}{3.90}{-2.48}{}
  \tomb{GX4}{3.90}{-3.10}{}
  \draw[die, draw=cGrLS] (P5) -- (GX1);
  \draw[die, draw=cTeaS] (P6) -- (GX2);
  \draw[die, draw=cYelS] (Q5) -- (GX3);
  \draw[die, draw=cGrDS] (Q6) -- (GX4);
  \draw[lin, draw=cBlueS] (P1) -- (Q1);
  \draw[lin, draw=cBlueS] (P1) to[out=-16,in=164] (Q2);
  \draw[lin, draw=cAmbS]  (P2) to[out=-14,in=166] (Q3);
  \draw[lin, draw=cAmbS]  (P2) to[out=-26,in=155] (Q4);
  \draw[lin, draw=cYelS]  (P3) to[out=-26,in=155] (Q5);
  \draw[lin, draw=cGrDS]  (P4) to[out=-26,in=155] (Q6);
  \draw[lin, draw=cBlueS] (Q1) -- (R1);
  \draw[lin, draw=cBlueS] (Q2) -- (R2);
  \draw[lin, draw=cBlueS] (Q2) to[out=-16,in=164] (R3);
  \draw[lin, draw=cAmbS]  (Q3) to[out=-16,in=164] (R4);
  \draw[lin, draw=cAmbS]  (Q4) to[out=-16,in=164] (R5);
  \draw[lin, draw=cAmbS]  (Q4) to[out=-26,in=155] (R6);
  \node[tagW] (T1) at (6.60, 0.00) {$a_1$};
  \node[tagW] (T2) at (6.60,-0.62) {$a_1$};
  \node[tagW] (T3) at (6.60,-1.24) {$a_1$};
  \node[tagW] (T4) at (6.60,-1.86) {$a_2$};
  \node[tagC] (T5) at (6.60,-2.48) {$a^{\star}$};
  \node[tagW] (T6) at (6.60,-3.10) {$a_2$};
  \draw[ans] (R1) -- (T1);
  \draw[ans] (R2) -- (T2);
  \draw[ans] (R3) -- (T3);
  \draw[ans] (R4) -- (T4);
  \draw[ans, draw=cTeaS!75] (R5) -- (T5);
  \draw[ans] (R6) -- (T6);
  \node[ulab] at (1.05,-3.56) {$R{=}6$};
  \node[ulab] at (3.25,-3.56) {$R{=}4$};
  \node[ulab, text=cDead] at (5.45,-3.56) {$R{=}2$};
  \node[ulab, text=cDead] at (6.60,-3.56) {$a^{\star}\!\times 1$};
  \node[coll] at (1.05,-4.10) {initial population};
  \node[coll] at (3.25,-4.10) {resampling 1};
  \node[coll] at (5.45,-4.10) {resampling 2};
  \node[coll] at (6.78,-4.10) {answers};
  \node[sub]  at (3.85,-4.66) {\textbf{(a)} Systematic resampling
    \;{\color{cMut}\itshape (weights reset to $1/N$)}};
\end{scope}
\draw[cMut!70, dash pattern=on 2.6pt off 2.2pt, line width=0.5pt] (7.45,0.40) -- (7.45,-4.30);
\begin{scope}[xshift=7.93cm]
  \bands
  \node[pb, minimum size=16pt] (p1) at (1.05, 0.00) {A};
  \node[pa, minimum size=13pt] (p2) at (1.05,-0.62) {B};
  \node[py, minimum size=11pt] (p3) at (1.05,-1.24) {C};
  \node[pd, minimum size= 9pt] (p4) at (1.05,-1.86) {D};
  \node[pl, minimum size= 8pt] (p5) at (1.05,-2.48) {E};
  \node[ring, minimum size=10.2pt] (p6) at (1.05,-3.10) {};
  \node[ph, minimum size=7pt] at (1.05,-3.10) {\tiny\textcolor{cTeaS}{$\star$}};
  \node[pb, minimum size=11.7pt] (q1) at (3.25, 0.00) {A};
  \node[pb, minimum size=11.7pt] (q2) at (3.25,-0.62) {A};
  \node[pa, minimum size=13pt]   (q3) at (3.25,-1.24) {B};
  \node[py, minimum size=11pt]   (q4) at (3.25,-1.86) {C};
  \node[pd, minimum size= 9pt]   (q5) at (3.25,-2.48) {D};
  \node[ring, minimum size=11.6pt] (q6) at (3.25,-3.10) {};
  \node[ph, minimum size=8.6pt] at (3.25,-3.10) {\tiny\textcolor{cTeaS}{$\star$}};
  \node[pb, minimum size=12.2pt] (r1) at (5.45, 0.00) {A};
  \node[pb, minimum size=11.2pt] (r2) at (5.45,-0.62) {A};
  \node[pa, minimum size=10.8pt] (r3) at (5.45,-1.24) {B};
  \node[pa, minimum size=10.8pt] (r4) at (5.45,-1.86) {B};
  \node[py, minimum size=10.4pt] (r5) at (5.45,-2.48) {C};
  \node[ring, minimum size=15pt] (r6) at (5.45,-3.10) {};
  \node[ph, minimum size=11.6pt] at (5.45,-3.10) {\scriptsize\textcolor{cTeaS}{$\star$}};
  \tomb{gx1}{1.70}{-2.48}{}
  \tomb{gx2}{3.90}{-2.48}{}
  \draw[die, draw=cGrLS] (p5) -- (gx1);
  \draw[die, draw=cGrDS] (q5) -- (gx2);
  \draw[lin, draw=cBlueS] (p1) -- (q1);
  \draw[lin, draw=cBlueS] (p1) to[out=-16,in=164] (q2);
  \draw[lin, draw=cAmbS]  (p2) to[out=-14,in=166] (q3);
  \draw[lin, draw=cYelS]  (p3) to[out=-14,in=166] (q4);
  \draw[lin, draw=cGrDS]  (p4) to[out=-14,in=166] (q5);
  \draw[lin, draw=cTeaS]  (p6) -- (q6);
  \draw[lin, draw=cBlueS] (q1) -- (r1);
  \draw[lin, draw=cBlueS] (q2) -- (r2);
  \draw[lin, draw=cAmbS]  (q3) -- (r3);
  \draw[lin, draw=cAmbS]  (q3) to[out=-16,in=164] (r4);
  \draw[lin, draw=cYelS]  (q4) to[out=-16,in=164] (r5);
  \draw[lin, draw=cTeaS]  (q6) -- (r6);
  \node[tagW] (t1) at (6.60, 0.00) {$a_1$};
  \node[tagC] (t2) at (6.60,-0.62) {$a^{\star}$};
  \node[tagW] (t3) at (6.60,-1.24) {$a_2$};
  \node[tagW] (t4) at (6.60,-1.86) {$a_3$};
  \node[tagC] (t5) at (6.60,-2.48) {$a^{\star}$};
  \node[tagC] (t6) at (6.60,-3.10) {$a^{\star}$};
  \draw[ans] (r1) -- (t1);
  \draw[ans, draw=cTeaS!75] (r2) -- (t2);
  \draw[ans] (r3) -- (t3);
  \draw[ans] (r4) -- (t4);
  \draw[ans, draw=cTeaS!75] (r5) -- (t5);
  \draw[ans, draw=cTeaS!75] (r6) -- (t6);
  \node[ulab] at (1.05,-3.56) {$R{=}6$};
  \node[ulab] at (3.25,-3.56) {$R{=}5$};
  \node[ulab, text=cTeaS] at (5.45,-3.56) {$R{=}4$};
  \node[ulab, text=cTeaS] at (6.60,-3.56) {$a^{\star}\!\times 3$};
  \node[coll] at (1.05,-4.10) {initial population};
  \node[coll] at (3.25,-4.10) {resampling 1};
  \node[coll] at (5.45,-4.10) {resampling 2};
  \node[coll] at (6.78,-4.10) {answers};
  \node[sub]  at (3.85,-4.66) {\textbf{(b)} Chopthin resampling
    \;{\color{cMut}\itshape (weights carried, ratio $\le \eta$)}};
\end{scope}
\end{tikzpicture}%
}}
\vspace{-20pt}
\caption{\textbf{One illustrative example of systematic and Chopthin resampling}
on the same toy population of $N=6$ weighted particles (circle area $\propto$ weight;
color/letter $=$ founding ancestor; $\times$ $=$ killed; $R$ $=$ surviving root
lineages). The $\star$ particle is low-weight but would reach the correct answer
$a^\star$; token-identical copies are merged before voting.
\textbf{(a)} In this realization, systematic resampling resets all weights to $1/N$:
$\star$ dies at the first event, $R$ collapses $6\to4\to2$ (\emph{particle
impoverishment}), and the majority vote fails even though the population contains
$a^\star$. \textbf{(b)} In this realization, Chopthin carries unequal weights forward:
$\star$ survives, $R$ falls only to $4$, and the vote is correct. Neither resampler
guarantees preservation of a trajectory that would eventually reach the correct answer;
the figure illustrates a possible mechanism, not an expected outcome. Here $\eta$ bounds
the ratio between the largest and smallest output weight
(Section~\ref{sec:chopthin}), and $a_1,a_2,a_3$ are distinct incorrect answers, with
equal tags denoting equal answers.}

\label{fig:genealogy}
\end{figure}
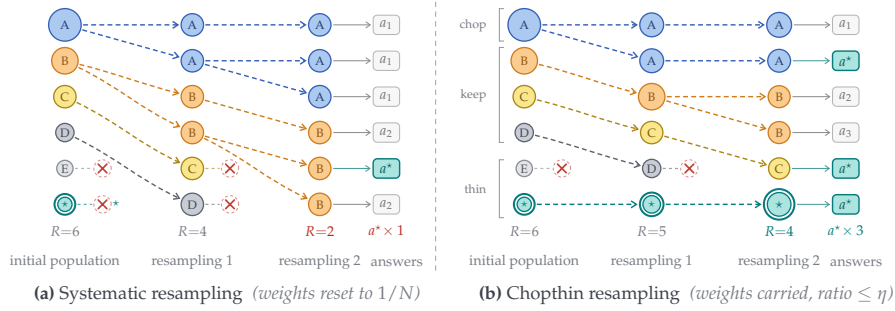

Our contributions are as follows.
\begin{itemize}[topsep=2pt,itemsep=2pt,parsep=0pt,leftmargin=1.4em]
    \item \textbf{We identify equal-weight resampling as a source of genealogical diversity loss}
    in SMC-based reasoning and distinguish this loss from the weight concentration summarized by
    the ESS.

    \item \textbf{We introduce the Chopthin resampler to language-model decoding}, bounding the
    ratio between the largest and smallest output weights and carrying the unequal weights forward,
    while preserving the SMC approximation in conditional expectation and guaranteeing a lower
    bound on post-resampling ESS.

    \item \textbf{We pair it with a deduplicate-then-cluster semantic-majority selector} that prevents exact duplicate final trajectories from receiving repeated count votes and clusters programs by execution behavior for code.

    \item \textbf{We evaluate across three open models and five reasoning benchmarks}: Chopthin raises oracle coverage in $13$ of $15$ settings, and CCPS matches or
    improves final-answer accuracy over Power-SMC in 14 of $15$, with ablations attributing the
    coverage gain to preserved trajectory diversity and the accuracy gain to the
    semantic-majority selector.
\end{itemize}

\section{Background: Power sampling as sequential Monte Carlo}
\label{sec:background}

\subsection{The sequence-level power distribution}

Let $p_\theta$ be a pretrained autoregressive language model over a vocabulary $\mathcal{V}$.
Given a prompt $x$, it assigns to a sequence $y=(y_1,\dots,y_T)$ terminating in an
end-of-sequence (EOS) token the probability
$p_\theta(y\mid x)=\prod_{t=1}^{T}p_\theta(y_t\mid x,y_{<t})$, where $y_{<t}:=(y_1,\dots,y_{t-1})$ and ``$:=$'' means ``is defined as''. For a sharpening exponent
$\alpha>1$, the \emph{sequence-level power distribution} is
\begin{equation}
\pi_\alpha(y \mid x) = \frac{p_\theta(y \mid x)^{\alpha}}{Z_\alpha(x)},\qquad
Z_\alpha(x)=\sum_{y} p_\theta(y \mid x)^{\alpha}.
\label{eq:power}
\end{equation}
The exponent acts on the probability of the complete sequence, moving mass toward
high-likelihood sequences without changing the model parameters. Recent analyses attribute much of the reasoning gain from RL post-training to this kind of sharpening
\citep{yue2025rlsharpening}. Exact sampling is intractable: the normalizer in \eqref{eq:power} sums over an exponentially large space of complete sequences. Sampling each token at temperature $1/\alpha$ renormalizes at every step, so the sequence distribution it induces is not $\pi_\alpha$
\citep{karan2026powersampling}.

\subsection{Sequential Monte Carlo over prefixes}

Power sampling can instead be written as Sequential Monte Carlo (SMC) \citep{azizi2026powersmc},
using the standard construction of \citep{delmoral2006sequential}. The sampler keeps $N$
particles. Particle $i$ holds a prefix $y^{(i)}_{1:t}$ and a weight, and all $N$ prefixes
are decoded together as one batch. After $t$ tokens, the unnormalized intermediate target is
\begin{equation}
\gamma_t(y_{1:t}\mid x)=p_\theta(y_{1:t}\mid x)^{\alpha_t},
\qquad
\alpha_t=1+(\alpha-1)\min\!\big(t/T_{\mathrm{ramp}},\,1\big),
\qquad
\alpha_0:=1.
\label{eq:prefix}
\end{equation}
Here $\gamma_t$ is the intermediate target after $t$ tokens: the same power construction on the first $t$ tokens, with
partial exponent $\alpha_t$. That exponent is the $\alpha$-ramp: it grows linearly
from $1$ to $\alpha$ over the first $T_{\mathrm{ramp}}$ tokens, then stays at $\alpha$. The
ramp mitigates weight degeneracy (the weight concentrating on a small number of particles)
early in decoding \citep{azizi2026powersmc}.

At each step, every particle samples its next token from a proposal $q(\cdot\mid x,y_{<t})$:
any distribution over the next token that depends only on the prompt and the tokens generated
so far. Each particle carries a cumulative unnormalized weight $U^{(i)}_t$, starting from
$U^{(i)}_0=1$. Let $\omega^{(i)}_t$ be the incremental importance weight of particle $i$ at
step $t$:
\begin{equation}
U^{(i)}_t=U^{(i)}_{t-1}\,\omega^{(i)}_t,
\qquad
\omega^{(i)}_t=\frac{\gamma_t\!\big(y^{(i)}_{1:t}\mid x\big)}
{\gamma_{t-1}\!\big(y^{(i)}_{1:t-1}\mid x\big)\,q\!\big(y^{(i)}_t\mid x,y^{(i)}_{<t}\big)}.
\label{eq:incr}
\end{equation}
Each increment connects one intermediate target to the next, so once the ramp
is over, the weighted population targets the power distribution over reachable
sequences. Appendix~\ref{app:powersmc} gives the expanded form of
\eqref{eq:incr}, explains the weight updates during the ramp, and describes
particles that finish early. A particle that emits EOS stops generating but
remains in the population; its token-level increments become one, while the
prefix-exponent correction continues during the ramp, so the population always
contains $N$ particles.

\subsection{Effective sample size and the resampling interface}

We monitor weight degeneracy using normalized weights and the
\emph{effective sample size} (ESS):
\begin{equation}
W^{(i)}_t=\frac{U^{(i)}_t}{\textstyle\sum_{j=1}^{N} U^{(j)}_t},\qquad
\mathrm{ESS}_t=\big(\textstyle\sum_{i=1}^N (W^{(i)}_t)^2\big)^{-1}\in[1,N],
\label{eq:ess}
\end{equation}
where $i$ indexes the particle whose normalized weight is being computed and $j$ is a
summation index running over the whole population. ESS is evaluated at block boundaries every $B$ decoded tokens, and the sampler resamples when $\mathrm{ESS}_t<\kappa N$; $\kappa$ is the trigger fraction.

At a \emph{resampling event}, write $w_i:=W^{(i)}_t$ for the normalized weight of input
particle $i$. The resampler returns $N$ output particles, indexed by $k$, together with ancestor
indices $A_{1:N}$, where $A_k\in\{1,\dots,N\}$, and output weights $w^{+}_{1:N}$, where $+$ denotes
immediately after resampling. Output particle $k$ continues the state of input particle $A_k$ and
takes cumulative weight $w^{+}_k$. The \emph{offspring count} of input particle $i$ is
$C_i=\sum_{k=1}^{N}\mathbf 1\{A_k=i\}$, so $C_i=0$ means particle $i$ is deleted. Multinomial,
stratified \citep{kitagawa1996monte}, residual \citep{liu1998sequential}, and systematic
resampling \citep{whitley1994genetic,carpenter1999improved} differ in resampling variance
\citep{douc2005comparison}, but all equalize the output weights: $w^{+}_k=1/N$. Our method changes this resampling rule while leaving the rest of the SMC construction unchanged.

\section{Motivation: weight balance is not trajectory diversity}
\label{sec:motivation}

\subsection{Weight balance and trajectory diversity are different objectives}
\label{sec:objectives}

The ESS \eqref{eq:ess} measures weight balance, not how many distinct trajectories survive: an equal-weight population maximizes ESS yet may consist of many copies of a few ancestors. We track this second axis with two quantities. The first is $R_t$, the number of surviving \emph{root lineages}: founding ancestors that still have a descendant at time $t$ (the genealogical diversity of the population). The second is $D$, the number of \emph{token-distinct complete trajectories} left after exact duplicates are merged at the end of decoding, so $D\le N$. The two are related but not interchangeable: separate lineages can converge to the same final trajectory, and descendants of one lineage can diverge. Token-distinct trajectories need not give distinct answers, but they preserve the opportunity for different solutions to emerge; answer-level effects are measured by our selector's clusters and the coverage metric of Section~\ref{sec:experiments}.
\subsection{Stochastic deletion and lineage collapse}
\label{sec:costs}

Equal-weight resampling solves weight degeneracy, and in classical filtering it works well \citep{douc2005comparison}. For preserving reasoning trajectories it has two costs.

First, \emph{stochastic deletion}. Under systematic resampling, a particle with $Nw_i<1$ receives either zero or one offspring, with average $Nw_i$, so it is deleted with probability $1-Nw_i$. Deleting by weight suits cases where current importance mass predicts future
usefulness; here it does not reliably predict whether the completed trajectory
will produce a correct answer, so resampling can delete trajectories that would
otherwise have reached a correct solution.

Second, \emph{lineage collapse}. A high-weight particle receives several offspring, all starting from the same prefix, so the number of distinct prefixes can only shrink at an event; the copies begin to differ only as new tokens are sampled. Repeated events can therefore leave many particles but few independent lineages (\emph{particle impoverishment}; \citealp{doucet2001sequential}): $R_t$ can fall at every event, as illustrated by one possible realization in Figure~\ref{fig:genealogy}a, and one equalizing event typically leaves far fewer unique ancestors than particles
\citep{azizi2026powersmc}.

Underlying both costs is a single move: equal-weight resampling converts each particle's continuous weight into an integer offspring count, discarding the weight information instead of carrying it forward.

\subsection{Design objective}
\label{sec:design}

We therefore seek a resampling rule that preserves the weighted target and the particle budget, guarantees control of post-resampling weight concentration, and avoids perturbing moderate particles or erasing accumulated weight information. Equal-weight resampling meets the first two demands, but it controls weight concentration in the strongest possible way (exact equality) and perturbs the whole population to get there. Exact equality is stronger than necessary: a bound on the ratio between the largest and smallest output weights already guarantees a floor on the ESS. This weaker condition defines our method.

Selection needs the same care: the standard rule draws one particle in proportion
to its final weight, so a correct trajectory that survives at a low weight is
rarely returned. The second component of our method is a selector that uses the
whole population instead (Section~\ref{sec:selection}).

\section{Chopthin-Consensus Power Sampling}
\label{sec:method}

\subsection{Chopthin preliminaries: bounded-weight resampling}
\label{sec:chopthin}

Chopthin \citep{gandy2016chopthin} is an existing resampler from the \emph{bounded-weight} family: instead of forcing equal output weights, it bounds how uneven they may be,
\begin{equation}
\max_{1\le k\le N} w^{+}_k \;\big/\; \min_{1\le k\le N} w^{+}_k
\;\le\;\eta,
\label{eq:ratio}
\end{equation}
where $\eta$ is a user-chosen ratio bound (for this resampler $\eta\ge4$; Appendix~\ref{app:chopthin} explains the constraint). A small $\eta$ forces the weights closer together; a large $\eta$ leaves them more uneven and perturbs the population less. All results in this subsection are from \citep{gandy2016chopthin}.
A bounded ratio is enough to control weight degeneracy, because it forces a floor on the ESS:
\begin{proposition}[ESS floor; {\citealp[Lemma~2]{gandy2016chopthin}}]
\label{prop:floor}
If the output weights $w^{+}_{1:N}>0$ satisfy \eqref{eq:ratio}, then
\begin{equation}
\mathrm{ESS}(w^{+})\;:=\;\frac{\big(\textstyle\sum_{k=1}^{N} w^{+}_k\big)^{2}}{\textstyle\sum_{k=1}^{N} (w^{+}_k)^{2}}
\;\ge\;\frac{4\big(\eta N+1-\eta^{2}\big)}{(\eta+1)^{2}}.
\label{eq:floor}
\end{equation}
\end{proposition}

Chopthin takes the input weights $w_{1:N}$ and $\eta$, computes a threshold $a>0$ from the current weights at every event ($a$ is not a user choice), and gives particle $i$ the expected offspring count $h^{\eta}_a(w_i)$, acting on each particle as shown:
\begin{equation}
h^{\eta}_{a}(w_i)\;=\;
\left\{
\begin{array}{lll}
w_i/a, & w_i<a & \text{\emph{thin}: survive with probability } w_i/a,\ \text{at weight } a,\\[2pt]
1, & a\le w_i<\eta a/2 & \text{\emph{keep}: one offspring, weight unchanged},\\[2pt]
2w_i/(\eta a), & w_i\ge \eta a/2 & \text{\emph{chop}: split into equal-weight pieces}.
\end{array}
\right.
\label{eq:h}
\end{equation}
The threshold is set so that the expected total matches the budget:
\begin{equation}
\textstyle\sum_{i=1}^{N} h^{\eta}_{a}(w_i)\;=\;N;
\label{eq:threshold}
\end{equation}
a solution exists because the total decreases as $a$ grows. Every output weight lies in $[a,\eta a]$, which gives \eqref{eq:ratio}. Particles in the keep band pass through unchanged, and if all input weights are equal, Chopthin returns every particle unchanged: the resampler perturbs the population only where the weights are uneven. The exact integer offspring counts and the weight-conservation correction are in Appendix~\ref{app:chopthin}.

The thin band is where Chopthin departs most from an equalizing scheme. The resampler returns $N$ particles, each of weight at least $a$, and conserves the total weight, so $Na\le\sum_k w^{+}_k=\sum_i w_i=1$ and hence $a\le 1/N$. A light particle therefore survives with probability $w_i/a\ge Nw_i$, exactly its survival probability under systematic resampling (Section~\ref{sec:costs}), so none is more likely to be deleted under Chopthin than under the equalizing baseline.

Chopthin keeps the statistical validity that Section~\ref{sec:design} asks for: it is \emph{unbiased} (on average, the offspring of particle $i$ carry exactly its weight $w_i$), returns exactly $N$ output particles, and conserves the total weight (Proposition~\ref{prop:guarantees}, Appendix~\ref{app:target}). Unbiasedness means the weighted target is unchanged in expectation, so Chopthin alters the resampling variance and genealogy, not the target (Proposition~\ref{prop:target}).

\subsection{Carried weights and integration}
\label{sec:integration}

CCPS \emph{carries} Chopthin's output weights forward, $U^{(k)}_t \leftarrow w^{+}_k$ for
$k=1,\dots,N$, instead of any uniform reset. SMC weights are defined only up to a common constant, so adopting the normalized outputs as the new cumulative weights is valid. Later increments \eqref{eq:incr} multiply the carried values, so the weights at the end of decoding keep information from before the last event.

Decoding runs as in Section~\ref{sec:background}. At a triggered event, Chopthin is called on
$W^{(1:N)}_t$, the particle states (sequences, termination flags, and cached model state) are
reordered by $A_{1:N}$, and the weights are carried forward as above.

Algorithm~\ref{alg:cps} (Appendix~\ref{app:algorithm}) lists the complete procedure. We set $\eta=3+\sqrt8\approx5.83$, chosen so that the asymptotic ESS floor is $N/2$, asymptotically matching the trigger fraction $\kappa=0.5$ (Section~\ref{sec:setup}); at our experimental $N=32$ the exact guaranteed floor is $13.17$ ($0.412N$). The derivation and finite-$N$ values are in Appendix~\ref{app:etachoice}.

\subsection{Semantic-majority selection}
\label{sec:selection}

A standard SMC rule is the \emph{weight draw}: pick one particle with probability proportional to its normalized final weight $\Wfinal{i}$ \eqref{eq:ess}, and return the answer parsed from it \citep{doucet2001sequential}. Particle-based decoding methods for language models inherit this rule \citep{azizi2026powersmc,nguyen2026apps}. CCPS instead aggregates over the whole population.

Selection proceeds in three steps: merge, cluster, and vote.

\emph{Step 1: Merge.} Combine token-identical final trajectories (for code, those with identical extracted programs) into a set $\mathcal T$ of distinct trajectories, so that each exact duplicate contributes one count vote; pool their final weights for tie-breaking. Each $u\in\mathcal T$ has an answer $a_u$ parsed from its trajectory (parsing details in Appendix~\ref{app:selector}).

\emph{Step 2: Cluster.} Group answers the task grader judges equivalent, comparing candidates only with one another, never with the gold answer. For code, we cluster programs by their exact behavioral signatures on a fixed ordered set of test inputs generated once per problem and base model and shared across both resampling arms. Call the resulting set of clusters $\mathcal C$.

\emph{Step 3: Vote.} Return the answer of the cluster supported by the largest number of distinct trajectories (distinct programs, for code),
\begin{equation}
C^{\ast}\;=\;\arg\max_{C\in\mathcal C}\;\big\lvert\{u\in\mathcal T : a_u\in C\}\big\rvert,
\qquad
\hat a_{\mathrm{maj}}\;=\;\mathrm{rep}(C^{\ast}),
\label{eq:maj}
\end{equation}
where $\mathrm{rep}(C)$ is the cluster representative of Appendix~\ref{app:selector-defs}, with ties between clusters broken by the sum of the pooled weights of their supporting trajectories, $\sum_{u\in\mathcal T:\,a_u\in C}\bar W_u$, where $\bar W_u=\sum_{i:\,y^{(i)}=u}\Wfinal{i}$. The output is a consensus point estimate, not a sample from the power distribution.

\section{Experiments}
\label{sec:experiments}

\subsection{Setup}
\label{sec:setup}

\paragraph{Models and benchmarks.} We evaluate Qwen2.5-Math-7B~\citep{yang2024qwen25math},
Qwen2.5-7B~\citep{qwen2024qwen25}, and Qwen3-4B~\citep{yang2025qwen3} on
MATH500~\citep{lightman2024verify}, GSM8K~\citep{cobbe2021gsm8k}, AIME 2022--2024~\citep{aimo2024aime},
GPQA Diamond~\citep{rein2023gpqa}, and HumanEval~\citep{chen2021humaneval} without additional fine-tuning.
Appendix~\ref{app:solver} gives model, benchmark, and prompting details.

\paragraph{Method settings.} The Power-SMC and Chopthin generation arms use $N=32$, $\alpha=2$,
proposal temperature $\tau=0.5$, and ESS trigger $\kappa=0.5$.
They use identical decoding configurations and differ only in the resampler;
the Chopthin arm uses $\eta=3+\sqrt8$. For final selection, Power-SMC uses
a weight draw, whereas CCPS uses semantic-majority selection. Complete
settings appear in Appendix~\ref{app:solver}.

\paragraph{Correctness and coverage.} An answer is \emph{correct} if it is equivalent to the gold
answer under the task-specific grader of Appendix~\ref{app:selector}. To measure what the final
population contains, independently of any selector, we report for each problem the coverage
indicator $c=\max_{i} g\big(a^{(i)},a^{\star}\big)$,
where $a^{\star}$ is the gold answer and $g\in\{0,1\}$ is the task grader; the \emph{oracle
coverage} is the average of $c$ over problems. Oracle coverage is the at-least-one-correct rate
over the correlated population of a single fixed-compute run. The final trajectories share
ancestry through resampling and can include exact duplicates, so this is not an i.i.d.\
pass@$N$. It is a ceiling for any selector that must return one of the final answers. For HumanEval, which has no
symbolic grader, we score a program correct when it passes the held-out tests; the selector
never sees these tests (Section~\ref{sec:selection}).

\subsection{Comparison study}

\begin{figure}[!t]
\centering
\begin{minipage}[t]{0.58\textwidth}
\vspace{0pt}
\centering
\resizebox{\linewidth}{!}{%
\renewcommand{\arraystretch}{1.12}\setlength{\tabcolsep}{3pt}%
\begin{tabular}[t]{llccccc}
\toprule
Model & Resampler $+$ selector & MATH500 & GSM8K & AIME & GPQA & HumanEval \\
\midrule
\multirow{4}{*}{Qwen2.5-Math-7B}
 & Systematic $+$ weight draw (Power-SMC) & 77.0 & 89.5 & 15.6 & 30.3 & 58.5 \\
 & Systematic $+$ majority & 79.4 & \textbf{91.5} & 15.6 & 31.3 & 61.0 \\
 & Chopthin $+$ weight draw & 78.0 & 88.8 & 15.6 & 31.8 & 58.5 \\
\rowcolor{oursgreen}
 & Chopthin $+$ majority (CCPS, ours) & \textbf{81.4} & 90.8 & \textbf{16.7} & \textbf{33.8} & \textbf{61.6} \\
\midrule
\multirow{4}{*}{Qwen2.5-7B}
 & Systematic $+$ weight draw (Power-SMC) & 74.2 & 91.0 & 10.4 & 29.3 & 73.2 \\
 & Systematic $+$ majority & \textbf{76.6} & \textbf{91.5} & 11.1 & \textbf{31.3} & 76.2 \\
 & Chopthin $+$ weight draw & 73.2 & 90.7 & 11.5 & 27.8 & 75.0 \\
\rowcolor{oursgreen}
 & Chopthin $+$ majority (CCPS, ours) & 76.0 & 91.4 & \textbf{12.2} & 30.3 & \textbf{76.8} \\
\midrule
\multirow{4}{*}{Qwen3-4B}
 & Systematic $+$ weight draw (Power-SMC) & 79.0 & 90.1 & 15.6 & 29.8 & \textbf{71.3} \\
 & Systematic $+$ majority & 81.0 & 91.7 & 15.6 & 34.3 & 70.7 \\
 & Chopthin $+$ weight draw & \textbf{81.8} & 89.8 & \textbf{16.7} & 34.3 & 70.7 \\
\rowcolor{oursgreen}
 & Chopthin $+$ majority (CCPS, ours) & \textbf{81.8} & \textbf{92.1} & 15.6 & \textbf{40.4} & 70.7 \\
\bottomrule
\end{tabular}%
}
\setlength{\abovecaptionskip}{2pt}\vspace{-3pt}
\captionof{table}{Selector ablation: final-answer accuracy (\%) for each
resampler $\times$ selector combination. Green rows
reproduce the corresponding entries of Table~\ref{tab:main}; bold marks
the best value per column within each model block.}
\label{tab:selector-ablation}
\end{minipage}\hfill
\begin{minipage}[t]{0.40\textwidth}
\vspace{0pt}
\centering
\definecolor{cTeaS}{RGB}{11,146,150}
\definecolor{cAmbS}{RGB}{208,124,26}
\definecolor{cBluS}{RGB}{52,96,182}

\definecolor{cInk}{RGB}{62,66,74}
\definecolor{cMut}{RGB}{120,124,132}

\definecolor{cZoneG}{HTML}{ECF4E6} %
\definecolor{cZoneR}{HTML}{FBEEEC} %

\centering

\tikzset{
  gdot/.style={
    circle,
    inner sep=0pt,
    minimum size=5.6pt,
    line width=0.8pt
  },
  gstem/.style={
    line width=1.2pt,
    line cap=round
  },
  glegenddot/.style={
    circle,
    inner sep=0pt,
    minimum size=6.2pt,
    line width=1.0pt
  },
  glegendstem/.style={
    line width=1.5pt,
    line cap=round
  }
}

\resizebox{\linewidth}{!}{%
\begin{tikzpicture}[
  font=\footnotesize,
  x=1.12cm,
  yscale=0.46
]

  \def\xmin{-2.2}
  \def\xmax{6.5}
  \def\ylo{0.0}
  \def\yhi{10.0}

  \fill[cZoneR] (\xmin,\ylo) rectangle (0,\yhi);
  \fill[cZoneG] (0,\ylo) rectangle (\xmax,\yhi);

  \draw[cInk, line width=0.8pt]
    (0,\ylo) -- (0,\yhi);

  \foreach \t/\tl in {-2/{\textminus2},0/{0},2/{2},4/{4},6/{6}}{
    \draw[cMut, line width=0.5pt]
      (\t,\ylo) -- (\t,{\ylo-0.16});

    \node[
      cInk,
      font=\footnotesize,
      anchor=north
    ] at (\t,{\ylo-0.24}) {\tl};
  }

  \node[
    black,
    font=\footnotesize,
    anchor=north
  ] at (2.2,{\ylo-0.7})
    {$\Delta$ oracle coverage (pp)};

  \foreach \yy/\name in {
    9/MATH500,
    7/GSM8K,
    5/AIME,
    3/GPQA,
    1/HumanEval
  }{
    \node[
      black,
      font=\small,
      anchor=east
    ] at ({\xmin-0.12},\yy) {\name};
  }

  \foreach \y/\val/\cs/\anc/\sh/\lab in {%
    9.55/2.4/cTeaS/west/0.24/{+2.4},
    9.00/1.8/cAmbS/west/0.24/{+1.8},
    8.45/1.0/cBluS/west/0.24/{+1.0},
    7.55/0.5/cTeaS/west/0.24/{+0.4},
    7.00/0.7/cAmbS/west/0.24/{+0.7},
    6.45/0.2/cBluS/west/0.24/{+0.2},
    5.55/4.4/cTeaS/west/0.24/{+4.4},
    5.00/0.0/cAmbS/west/0.24/{0.0},
    4.45/2.2/cBluS/west/0.24/{+2.2},
    3.55/4.0/cTeaS/west/0.24/{+4.0},
    3.00/-1.5/cAmbS/east/-0.24/{\textminus1.5},
    2.45/5.6/cBluS/west/0.24/{+5.6},
    1.55/4.3/cTeaS/west/0.24/{+4.3},
    1.00/1.3/cAmbS/west/0.24/{+1.3},
    0.45/1.2/cBluS/west/0.24/{+1.2}
  }{
    \draw[gstem, \cs]
      (0,\y) -- (\val,\y);

    \node[
      gdot,
      draw=\cs,
      fill=\cs!25!white
    ] at (\val,\y) {};

    \node[
      black,
      font=\scriptsize,
      anchor=\anc
    ] at ({\val+\sh},\y) {\lab};
  }

\end{tikzpicture}%
}
\setlength{\abovecaptionskip}{2pt}\vspace{-6pt}
\captionof{figure}{Oracle-coverage difference between Chopthin and systematic resampling across all settings. Color encodes the model: \textcolor{cTeaS}{Qwen2.5-Math-7B}, \textcolor{cAmbS}{Qwen2.5-7B}, \textcolor{cBluS}{Qwen3-4B}.}
\label{fig:gain}
\end{minipage}
\end{figure}
\begin{table}[!tp]
\centering
\caption{Final-answer accuracy (\%) for CCPS against Power-SMC and
training-free decoding baselines.}
\label{tab:main}
\resizebox{\textwidth}{!}{%
\begin{tabular}{llccccc}
\toprule
Model & Method & MATH500 & GSM8K & AIME & GPQA & HumanEval \\
\midrule

\multirow{5}{*}{Qwen2.5-Math-7B}
 & Baseline decoding
 & 65.0 & 80.0 & 11.1 & 20.7 & 20.7 \\

 & Low-temperature decoding ($\tau{=}1/\alpha$)
 & 69.8 & 88.2 & 8.9 & 32.8 & 30.5 \\

 & MH power sampling \citep{karan2026powersampling}
 & $74.8^{\dagger}$
 & $81.5^{\ddagger}$
 & $9.5^{\text{\S}}$
 & $\mathbf{38.9}^{\dagger}$
 & $57.3^{\dagger}$ \\

 & Power-SMC (systematic)
 & 77.0 & 89.5 & 15.6 & 30.3 & 58.5 \\

\rowcolor{oursgreen}
 & \textbf{CCPS (ours)}
 & \textbf{81.4}
 & \textbf{90.8}
 & \textbf{16.7}
 & 33.8
 & \textbf{61.6} \\

\midrule

\multirow{5}{*}{Qwen2.5-7B}
 & Baseline decoding
 & 50.8 & 82.2 & 1.1 & 30.3 & 28.0 \\

 & Low-temperature decoding ($\tau{=}1/\alpha$)
 & 61.2 & 88.4 & 6.7 & 29.3 & 7.9 \\

 & MH power sampling \citep{karan2026powersampling}
 & $70.6^{\dagger}$
 & $84.5^{\ddagger}$
 & $8.2^{\text{\S}}$
 & $\mathbf{31.8}^{\dagger}$
 & $62.2^{\dagger}$ \\

 & Power-SMC (systematic)
 & 74.2 & 91.0 & 10.4 & 29.3 & 73.2 \\

\rowcolor{oursgreen}
 & \textbf{CCPS (ours)}
 & \textbf{76.0}
 & \textbf{91.4}
 & \textbf{12.2}
 & 30.3
 & \textbf{76.8} \\

\midrule

\multirow{5}{*}{Qwen3-4B}
 & Baseline decoding
 & \textbf{82.0}
 & 91.8
 & 18.9
 & 39.9
 & 63.4 \\

 & Low-temperature decoding ($\tau{=}1/\alpha$)
 & 80.4
 & 91.2
 & \textbf{21.1}
 & 36.4
 & 55.5 \\

 & MH power sampling \citep{karan2026powersampling}
 & $76.2^{\ddagger}$
 & $90.6^{\ddagger}$
 & $9.5$
 & $36.4^{\ddagger}$
 & $45.7^{\ddagger}$ \\

 & Power-SMC (systematic)
 & 79.0
 & 90.1
 & 15.6
 & 29.8
 & \textbf{71.3} \\

\rowcolor{oursgreen}
 & \textbf{CCPS (ours)}
 & 81.8
 & \textbf{92.1}
 & 15.6
 & \textbf{40.4}
 & 70.7 \\

\bottomrule
\end{tabular}%
}
\\[3pt]
{\footnotesize
Bold numeric entries mark the highest reported value in each
model--benchmark block. Green rows indicate our method.
$^{\dagger}$Reported by \citep{karan2026powersampling} ($\alpha = 4$).
$^{\ddagger}$Reported by \citep{azizi2026powersmc} ($\alpha = 4$).
$^{\text{\S}}$Reported by \citep{zhou2026cutting} on AIME.}
\end{table}
Chopthin raises \emph{oracle coverage}: Figure~\ref{fig:gain} shows it lifts the ceiling over Power-SMC
in thirteen of the fifteen cells, ties in one, and dips in only one, by at most $1.5$ percentage points (absolute values in Table~\ref{tab:coverage}, Appendix~\ref{app:coverage}). The semantic-majority selector
then converts that headroom into accuracy: together they match or improve final-answer accuracy
over the full Power-SMC pipeline (systematic resampling $+$ weight draw) in $14$ of $15$ cells, by up to $10.6$ points (Qwen3-4B on GPQA); the single exception is Qwen3-4B on HumanEval ($-0.6$). Table~\ref{tab:main} reports these
results alongside Power-SMC and training-free decoding baselines.

Table~\ref{tab:selector-ablation} reports the full resampler $\times$
selector grid. No single selector is best in every cell, but the pattern
is selector-robust: semantic majority matches or improves on the weight
draw under \emph{both} resamplers in $28$ of $30$ comparisons (the
exceptions are Qwen3-4B on HumanEval under systematic resampling and on
AIME under Chopthin). Chopthin's contribution, by contrast, is the
\emph{ceiling}: it places a correct answer among the $N$ trajectories
more often (Figure~\ref{fig:gain}), and with the selector held at
majority it wins $9$ of $15$ cells, ties $2$, and loses $4$.
Baseline decoding (temperature $1$)
and low-temperature decoding ($\tau = 1/\alpha$) are single-sample references with no resampling.
The Metropolis--Hastings (MH) row collects reference values reported at $\alpha = 4$ in the
literature, marked per source in the table footnote. We do not tabulate methods available
only at other sampling configurations or without public code
\citep{ji2026scalable,nguyen2026apps}. Reported GRPO reference values are collected in
Appendix~\ref{app:grpo} \citep{shao2024deepseekmath,nguyen2026apps}.

\subsection{Discussion}
\label{sec:analysis}

\paragraph{Resampling and selection play distinct roles.} A higher ceiling is not automatically a better answer. Under the
weight draw, Chopthin's final-answer accuracy tracks Power-SMC's and sometimes falls below it
despite higher oracle coverage: the weights the draw follows do not, on these benchmarks,
favor the correct trajectories. The semantic-majority selector counts distinct
trajectories instead of trusting the weights, which is what turns the higher coverage into the gains of Table~\ref{tab:main}. Chopthin's more consistent effect is on oracle coverage.
Qwen3-4B on MATH500 shows the limit: baseline decoding scores $82.0$, above Power-SMC's $79.0$
and CCPS's $81.8$ (Table~\ref{tab:main}), yet CCPS oracle coverage is $84.8$, the final
population holds correct trajectories the current selector does not return. \citet{arzhantsev2026marginal} report a similar pattern for Metropolis--Hastings.

\section{Ablations}
\label{sec:ablations}

Chopthin has two degrees of freedom: the ratio bound $\eta$, and the choice to carry the
unequal output weights forward instead of resetting them. We ablate both on the primary cell
(Qwen2.5-Math-7B on MATH500), using the configuration of Section~\ref{sec:setup}.

\subsection{Sensitivity to the ratio bound $\eta$}
\label{sec:ablation-eta}

The bound is set by inverting the ESS floor of Proposition~\ref{prop:floor} at a target
fraction $\rho$ (Appendix~\ref{app:etachoice}), so we sweep $\rho$ and report the induced $\eta$ (Figure~\ref{fig:eta-ablation}), with the trigger fixed at
$\kappa = 0.5$. All arms are identical to CCPS except $\eta$.

\begin{figure}[!t]
\providecommand{\etafigscale}{1.0}
\definecolor{cTeaS}{RGB}{6,122,124}
\definecolor{cTeaF}{RGB}{172,228,223}
\definecolor{cInk}{RGB}{62,66,74}
\definecolor{cMut}{RGB}{120,124,132}
\definecolor{cMaj}{RGB}{52,96,182}
\definecolor{cWin}{RGB}{34,120,52}
\definecolor{cBandG}{HTML}{E2EFDA}
\centering
\begin{minipage}[b]{0.56\textwidth}\centering
\resizebox{\linewidth}{!}{%
\begin{tikzpicture}[font=\small, xscale=1.45, yscale=0.86]

\fill[cBandG] (1.50,-0.02) rectangle (2.40,4.62);

\draw[cMut, line width=0.5pt] (0,1.75) -- (0,4.55);
\draw[cMut, line width=0.5pt] (0,1.75) -- (6.8,1.75);
\foreach \v/\y in {78/1.75, 82/2.87, 86/3.99}{
  \draw[cMut, line width=0.5pt] (0,\y) -- (-0.055,\y);
  \node[black, font=\small, anchor=east] at (-0.10,\y) {\v};
  \draw[dotted, black!14, line width=0.5pt] (0,\y) -- (6.8,\y);}
\node[black, font=\footnotesize, rotate=90, align=center] at (-1.02,3.15) {Accuracy /\\coverage (\%)};

\draw[cInk, line width=0.8pt]
  (0.5,4.27) -- (1.95,4.27) -- (3.4,3.99) -- (4.85,3.822) -- (6.3,4.158);
\foreach \i/\y in {0/4.27, 1/4.27, 2/3.99, 3/3.822, 4/4.158}{
  \fill[cInk] ({0.5+\i*1.45-0.038},{\y-0.055}) rectangle ({0.5+\i*1.45+0.038},{\y+0.055});}
\node[cInk, font=\footnotesize, anchor=south] at (1.95,4.33) {87.0};
\node[cInk, font=\footnotesize, anchor=north] at (4.85,3.75) {85.4};

\draw[cMaj, line width=1.1pt]
  (0.5,2.31) -- (1.95,2.702) -- (3.4,2.198) -- (4.85,2.198) -- (6.3,2.646);
\foreach \i/\y in {0/2.31, 2/2.198, 3/2.198, 4/2.646}{
  \node[circle, fill=cMaj, inner sep=1.9pt] at ({0.5+\i*1.45},\y) {};}
\node[circle, fill=cWin, inner sep=3.0pt] at (1.95,2.702) {};
\node[circle, draw=cWin, line width=0.6pt, inner sep=4.5pt] at (1.95,2.702) {};
\node[cWin, font=\normalsize\bfseries, anchor=south] at (1.95,2.93) {81.4\%};
\node[cMaj, font=\footnotesize, anchor=north] at (0.5,2.24) {80.0};
\node[cMaj, font=\footnotesize, anchor=north] at (6.3,2.575) {81.2};

\node[circle, fill=cMaj, inner sep=1.9pt] at (1.30,5.10) {};
\node[cInk, font=\footnotesize, anchor=west] at (1.40,5.10) {Majority acc.};
\fill[cInk] (3.95,5.045) rectangle (4.026,5.155);
\node[cInk, font=\footnotesize, anchor=west] at (4.10,5.10) {Oracle coverage};

\draw[cMut, line width=0.5pt] (0,0) -- (0,1.0);
\draw[cMut, line width=0.5pt] (0,0) -- (6.8,0);
\foreach \d/\y in {12/0, 15/0.5, 18/1.0}{
  \draw[cMut, line width=0.5pt] (0,\y) -- (-0.055,\y);
  \node[black, font=\footnotesize, anchor=east] at (-0.10,\y) {\d};
  \draw[dotted, black!14, line width=0.5pt] (0,\y) -- (6.8,\y);}
\node[black, font=\footnotesize] at (-0.60,0.5) {$D$};

\draw[cTeaS, line width=0.8pt, dash pattern=on 2.2pt off 1.6pt]
  (0.5,0.45) -- (1.95,0.533) -- (3.4,0.567) -- (4.85,0.667) -- (6.3,0.867);
\foreach \i/\y/\d in {0/0.45/14.7, 1/0.533/15.2, 2/0.567/15.4, 3/0.667/16.0, 4/0.867/17.2}{
  \node[circle, draw=cTeaS, fill=cTeaF, line width=0.7pt, inner sep=1.7pt] at ({0.5+\i*1.45},\y) {};
  \node[cMut, font=\footnotesize, anchor=south] at ({0.5+\i*1.45},{\y+0.10}) {\d};}

\foreach \i/\e/\r in {0/{$4$}/{0.64}, 1/{$3{+}\sqrt8$}/{0.50}, 2/{$7.87$}/{0.40}, 3/{$11.24$}/{0.30}, 4/{$24.99$}/{0.148}}{
  \draw[cMut, line width=0.5pt] ({0.5+\i*1.45},0) -- ({0.5+\i*1.45},-0.08);
  \node[black, font=\footnotesize, anchor=north] at ({0.5+\i*1.45},-0.14) {\e};
  \node[cMut, font=\footnotesize, anchor=north] at ({0.5+\i*1.45},-0.62) {$\rho{=}\r$};}
\node[black, font=\footnotesize] at (3.4,-1.5) {ratio bound $\eta$};
\end{tikzpicture}}
\\[2pt]{\footnotesize (a) sensitivity to $\eta$}
\end{minipage}\hfill
\begin{minipage}[b]{0.415\textwidth}\centering
\resizebox{\linewidth}{!}{%
\begin{tikzpicture}[font=\small, xscale=1.6, yscale=0.86]

\fill[cBandG] (3.15,-0.02) rectangle (4.05,4.62);

\draw[cMut, line width=0.5pt] (0,1.75) -- (0,4.55);
\draw[cMut, line width=0.5pt] (0,1.75) -- (4.1,1.75);
\foreach \v/\y in {78/1.75, 82/2.87, 86/3.99}{
  \draw[cMut, line width=0.5pt] (0,\y) -- (-0.06,\y);
  \node[black, font=\small, anchor=east] at (-0.11,\y) {\v};
  \draw[dotted, black!14, line width=0.5pt] (0,\y) -- (4.1,\y);}
\node[black, font=\footnotesize, rotate=90, align=center] at (-0.98,3.15) {Accuracy /\\coverage (\%)};

\draw[cInk, line width=0.8pt] (0.5,3.598) -- (2.05,3.934) -- (3.6,4.27);
\foreach \x/\y in {0.5/3.598, 2.05/3.934, 3.6/4.27}{
  \fill[cInk] ({\x-0.0344},{\y-0.055}) rectangle ({\x+0.0344},{\y+0.055});}
\node[cInk, font=\footnotesize, anchor=south] at (0.5,3.66) {84.6};
\node[cInk, font=\footnotesize, anchor=south] at (2.05,3.99) {85.8};
\node[cInk, font=\footnotesize\bfseries, anchor=south] at (3.6,4.33) {87.0};

\draw[cMaj, line width=1.1pt] (0.5,2.142) -- (2.05,2.478) -- (3.6,2.702);
\foreach \x/\y in {0.5/2.142, 2.05/2.478}{
  \node[circle, fill=cMaj, inner sep=1.9pt] at (\x,\y) {};}
\node[circle, fill=cWin, inner sep=3.0pt] at (3.6,2.702) {};
\node[circle, draw=cWin, line width=0.6pt, inner sep=4.5pt] at (3.6,2.702) {};
\node[cWin, font=\normalsize\bfseries, anchor=south] at (3.6,2.93) {81.4\%};

\node[cMaj, font=\footnotesize, anchor=south] at (0.5,2.23) {79.4};
\node[cMaj, font=\footnotesize, anchor=south] at (2.05,2.57) {80.6};

\node[circle, fill=cMaj, inner sep=1.9pt] at (0.15,5.10) {};
\node[cInk, font=\footnotesize, anchor=west] at (0.27,5.10) {Majority acc.};
\fill[cInk] ({1.95-0.0344},5.045) rectangle ({1.95+0.0344},5.155);
\node[cInk, font=\footnotesize, anchor=west] at (2.05,5.10) {Oracle coverage};

\draw[cMut, line width=0.5pt] (0,0) -- (0,1.0);
\draw[cMut, line width=0.5pt] (0,0) -- (4.1,0);
\foreach \d/\y in {12/0, 15/0.5, 18/1.0}{
  \draw[cMut, line width=0.5pt] (0,\y) -- (-0.06,\y);
  \node[black, font=\footnotesize, anchor=east] at (-0.11,\y) {\d};
  \draw[dotted, black!14, line width=0.5pt] (0,\y) -- (4.1,\y);}
\node[black, font=\footnotesize] at (-0.64,0.5) {$D$};

\draw[cTeaS, line width=0.8pt, dash pattern=on 2.2pt off 1.6pt]
  (0.5,0.167) -- (2.05,0.533) -- (3.6,0.533);
\foreach \x/\y in {0.5/0.167, 2.05/0.533, 3.6/0.533}{
  \node[circle, draw=cTeaS, fill=cTeaF, line width=0.7pt, inner sep=1.7pt] at (\x,\y) {};}
\node[cMut, font=\footnotesize, anchor=south] at (0.5,0.267) {13.0};
\node[cMut, font=\footnotesize\bfseries, anchor=south] at (2.05,0.633) {15.2};
\node[cMut, font=\footnotesize\bfseries, anchor=south] at (3.6,0.633) {15.2};

\foreach \i/\lab in {0/{Systematic}, 1/{Chopthin\\$+$ reset}, 2/{Chopthin\\carried (ours)}}{
  \draw[cMut, line width=0.5pt] ({0.5+\i*1.55},0) -- ({0.5+\i*1.55},-0.08);
  \node[black, font=\footnotesize, anchor=north, align=center] at ({0.5+\i*1.55},-0.14) {\lab};}
\node[black, font=\footnotesize] at (2.05,-1.5) {resampler};
\end{tikzpicture}}
\\[2pt]{\footnotesize (b) carried weights vs.\ uniform reset}
\end{minipage}
\caption{\textbf{Ablations on the primary cell} (Qwen2.5-Math-7B, MATH500).
\textbf{(a)} The green marker is the default $\eta = 3+\sqrt8$. Top:
final-answer accuracy of CCPS (blue) and oracle coverage (dark squares); bottom: distinct final
trajectories $D$ out of $N$. Accuracy varies within $1.4$ points and coverage within $1.6$
across the range. \textbf{(b)} The hybrid
applies Chopthin's allocation \eqref{eq:h} then resets every weight to $1/N$: a diagnostic, not
a valid resampler. Bold marks the best value per column.}
\label{fig:eta-ablation}
\end{figure}
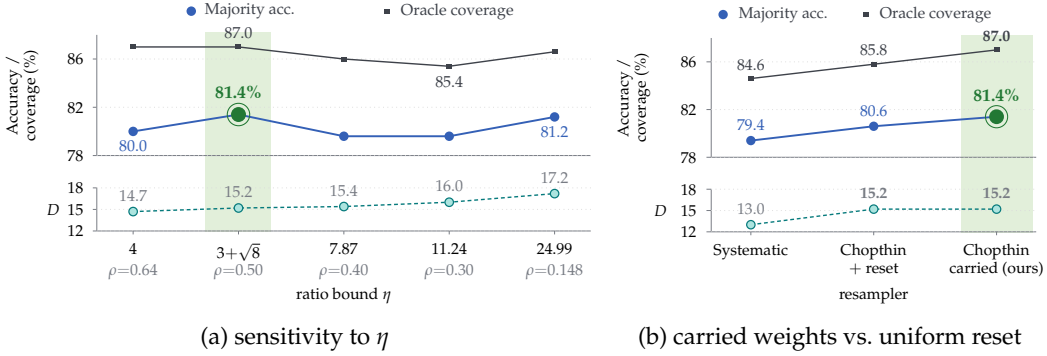

The method is robust to its one free parameter: every value of
$\eta$ keeps oracle coverage above the systematic baseline of Figure~\ref{fig:eta-ablation}(b) ($84.6$,
a margin of $+0.8$ to $+2.4$). A looser bound perturbs the population less and preserves
more distinct trajectories, which is the mechanism the method relies on. The trigger-matched
default $\rho = 0.5$ attains the best majority accuracy, so the theoretically motivated choice
of Section~\ref{sec:chopthin} is also the empirical optimum; we tuned $\eta$ no further.

\subsection{Carried weights versus uniform reset}
\label{sec:ablation-carry}

Chopthin changes both who survives a resampling event (probabilistic survival of low-weight
particles and the untouched keep band) and what the survivors remember (the carried weights). A hybrid that
applies Chopthin's offspring allocation \eqref{eq:h} but resets every output weight to
$1/N$ disentangles the two. Chopthin's unbiasedness holds for the product of offspring count and output weight (Eq.~\ref{eq:unbiased}), so resetting the weights to $1/N$ without adjusting the counts breaks that unbiasedness. Figure~\ref{fig:eta-ablation}(b) reports the comparison.

In this diagnostic, roughly half of the observed gain over the equalizing
baseline comes from the offspring allocation alone: the hybrid preserves the
same mean number of distinct trajectories as carried-weight Chopthin ($15.2$),
compared with $13.0$ under systematic resampling, and lifts both final-answer
accuracy and oracle coverage. Carrying the unequal weights adds a further
increment ($+0.8$ majority accuracy and $+1.2$ oracle coverage).

\section{Related work}
\label{sec:related}

\paragraph{Particle-based inference-time methods for LLMs.}
Twisted SMC \citep{zhao2024twistedsmc}, SMC steering and controlled generation
\citep{lew2023smc,loula2025smcconstraints}, reward-guided particle filtering
\citep{puri2025reasoningsmc}, and SMC power sampling \citep{azizi2026powersmc} all re-equalize
the particle weights at every resampling event.\footnote{The nearest exception is the without-replacement scheme of \citep{lew2023smc}, related to \citep{fearnhead2003resampling}.} To our
knowledge, no existing particle-based
inference-time method for language models bounds post-resampling weight unevenness or guarantees
a post-resampling effective sample size. Recent diversity-preserving efforts act \emph{around} an
equalizing resampler rather than replacing it, by tempering the weights before resampling
\citep{giannone2025epf}, adding rollout- or value-based auxiliary weights \citep{nguyen2026apps},
weighting by intrinsic uncertainty \citep{giannone2026ipf}, enlarging the eligible pool with
reward-scored historical prefixes \citep{tran2026pbsmc}, or restoring diversity after an
equalizing resample with Metropolis--Hastings rejuvenation moves \citep{markovic2026smc}.
These approaches change \emph{which} particles are favored; we change \emph{how much} the
population is perturbed, with no extra rollouts, learned potentials, or reward models.

\paragraph{Answer selection.}
Self-consistency returns the most common answer among independent samples
\citep{wang2023selfconsistency}; universal self-consistency and semantic-uncertainty methods
cluster free-form answers by meaning before aggregating
\citep{chen2023usc,kuhn2023semanticentropy}. For code, AlphaCode clusters programs by their
behavior on generated inputs \citep{li2022alphacode}, and CodeT selects candidates by agreement
on model-generated tests \citep{chen2023codet}. \citep{brown2024monkeys} show that the coverage
of repeated sampling grows with the number of samples while common selectors often fail to
realize it. Most recently, marginal sharpening \citep{arzhantsev2026marginal} makes the
sharpened marginal distribution over answers itself the inference-time target, with
self-consistency as its limiting case. We keep Power-SMC's sequence-level target and act on the resampler and selector instead, so the
approaches are complementary; our selector adapts these ideas to a resampled population, where
duplicates must be merged before any vote is meaningful (Section~\ref{sec:selection}).

\section{Conclusion}

We identified systematic resampling as an obstacle to reasoning performance in
SMC-based power sampling. It equalizes particle weights and drops low-weight
trajectories that, though unlikely early, can still reach the correct answer.
Replacing it with Chopthin gives low-weight trajectories a greater chance of
remaining represented and increases observed oracle coverage in $13$ of $15$
settings across three open models and five benchmarks. Pairing Chopthin with
the semantic-majority selector matches or improves final-answer accuracy over
the Power-SMC baseline in $14$ of $15$ settings, by up to $10.6$ percentage
points. Because the resampling change is independent of how tokens are proposed
and operates at the standard resampling interface, it may be combined with
other SMC-based decoding methods. Evaluating such combinations is future work.

\section*{LLM usage}
LLMs were used only for editing assistance; the authors take full responsibility for all content.

\bibliography{references_chopthin}
\bibliographystyle{colm2026_conference}

\appendix

\section{Details of the Chopthin resampler}
\label{app:chopthin}

The
main text uses only Chopthin's interface (ancestor indices and output weights), the three-band
behavior \eqref{eq:h}, and the guarantees of Propositions~\ref{prop:guarantees}
and~\ref{prop:floor}; the details here are what realize them.

\subsection{Expected-offspring function and threshold}
\label{app:h}

\begin{fromchopthin}

The expected-offspring function $h^{\eta}_a$ of \eqref{eq:h} is
valid for $\eta\ge4$; the map $a\mapsto\sum_i h^{\eta}_a(w_i)$ is continuous and non-increasing,
diverges as $a\to0^+$, and vanishes as $a\to\infty$, so a solution of the threshold equation
\eqref{eq:threshold} exists. The threshold partitions the particles
into \emph{light} particles $L = \{i : w_i < a\}$, which are thinned, and \emph{heavy} particles
$H = \{i : w_i \geq a\}$, which are kept or chopped; particles with $a \leq w_i < \eta a/2$ have
$h^{\eta}_a(w_i) = 1$ and pass through unchanged.
\end{fromchopthin}

When the solution set of \eqref{eq:threshold} is an interval, our implementation takes its
upper endpoint, so boundary particles with $w_i=a$ fall on the heavy side; the
equal-weights property of Section~\ref{sec:chopthin} relies on this convention. In
particular, for equal weights the upper endpoint is $a=w_i$: every particle falls in the keep
band with expected offspring count one, deterministically, which yields the property.

\subsection{Thinning and chopping}
\label{app:thinchop}

\begin{fromchopthin}
\paragraph{Thinning.} A light particle has $h^{\eta}_a(w_i) = w_i/a < 1$ and receives zero or one
offspring. The survivors are drawn by a single systematic pass over the expected counts
$\{w_i/a\}_{i \in L}$, so particle $i$ survives with probability $w_i/a$, and a survivor is assigned the
weight $a$. Let $N_L$ be the number of survivors.

\paragraph{Chopping.} A heavy particle receives $C_i = \lfloor h^{\eta}_a(w_i) \rfloor + m_i$ offspring:
the integer parts are guaranteed, and the residual $N_U = N - N_L - \sum_{i \in H} \lfloor
h^{\eta}_a(w_i) \rfloor$ offspring are allocated by a second systematic draw over the fractional parts
$\{ h^{\eta}_a(w_i) \}$, giving $m_i \in \{0, 1\}$ (here $\{x\} := x - \lfloor x\rfloor$ denotes
the fractional part). Because $N_L$ is random, the weight transferred to
the light side deviates from its expectation; the constant
\[
  \zeta = \frac{\sum_{i \in L} w_i - a N_L}{\sum_{i \in H} \{ h^{\eta}_a(w_i) \}}
\]
redistributes this deviation over the heavy particles. A heavy particle with $C_i = c > 0$ is split
into $c$ equal pieces, each of weight $\hat{w}_i / c$ with $\hat{w}_i = w_i + \zeta \{ h^{\eta}_a(w_i)
\}$. Since $\mathbb{E}[\zeta] = 0$ the step stays unbiased, and the correction makes the total weight
\emph{exactly} conserved in exact arithmetic (weight conservation,
Proposition~\ref{prop:guarantees}); in implementation, conservation holds to numerical
tolerance. Every output weight
lies in $[a, \eta a]$, which gives the bounded ratio \eqref{eq:ratio}.
\end{fromchopthin}

\subsection{Pseudocode}
\label{app:pseudocode}

Algorithm~\ref{alg:chopthin} assembles the steps above; it is adapted from
\citep{gandy2016chopthin}, in our notation and with an explicit guard for the case in which
every heavy expected count is an integer (then $N_U=0$ and $\zeta$ is set to zero).

\begin{algorithm}[t]
\caption{The Chopthin resampler, adapted from \citep{gandy2016chopthin}.}
\label{alg:chopthin}
{\glcolor
\begin{algorithmic}[1]
  \Require weights $w_{1:N}$ (sum $> 0$), ratio bound $\eta \geq 4$, target count $N$
  \Ensure ancestors $A_{1:N}$, weights $w^{+}_{1:N}$ with $\max_k w^{+}_k / \min_k w^{+}_k \leq \eta$
  \State solve $\sum_i h^{\eta}_a(w_i) = N$ for $a$ \Comment{upper endpoint}
  \State $L \gets \{i : w_i < a\}$; \quad $H \gets \{i : w_i \geq a\}$
  \State draw $u \sim \mathcal{U}(0, 1)$ \Comment{thin: systematic over $\{w_i/a\}_{i \in L}$}
  \For{$i \in L$}
    \State $u \gets u + w_i/a$
    \If{$u \geq 1$}
      \State emit $(A = i,\ w^{+} = a)$; \quad $u \gets u - 1$
    \EndIf
  \EndFor
  \State $N_L \gets$ number of survivors
  \State $N_U \gets N - N_L - \sum_{i \in H} \lfloor h^{\eta}_a(w_i) \rfloor$
  \If{$\sum_{i \in H} \{ h^{\eta}_a(w_i) \} = 0$}\Comment{all heavy counts integral; then $N_U=0$}
    \State $\zeta \gets 0$;\quad $m_i \gets 0$ for all $i \in H$
  \Else
    \State $\zeta \gets \big( \sum_{i \in L} w_i - a N_L \big) \big/ \sum_{i \in H} \{ h^{\eta}_a(w_i) \}$
    \State allocate $N_U$ residual offspring $m_i \in \{0,1\}$ by a systematic draw over $\{ h^{\eta}_a(w_i) \}$, $i \in H$
  \EndIf
  \For{$i \in H$}
    \State $c \gets \lfloor h^{\eta}_a(w_i) \rfloor + m_i$
    \If{$c > 0$}
      \State emit $c$ copies of $\big(A = i,\ w^{+} = (w_i + \zeta \{ h^{\eta}_a(w_i) \})/c \big)$
    \EndIf
  \EndFor
\end{algorithmic}}
\end{algorithm}

\subsection{Why $\eta\ge4$}
\label{app:eta4}

\begin{fromchopthin}
The condition $\eta\ge4$ arises from requiring every admissible split to keep its pieces inside
$[a,\eta a]$: the constraint $2w/(\eta a)\le\lfloor w/a\rfloor$ for all $w\ge a$, evaluated for
$w$ just below $2a$, forces $\eta\ge4$ \citep{gandy2016chopthin}.
\end{fromchopthin}

\subsection{Target preservation}
\label{app:target}
\begin{proposition}[Chopthin guarantees]
\label{prop:guarantees}
Given the input weights: \emph{unbiasedness},
\begin{equation}
\mathbb E\Big[\textstyle\sum_{k:\,A_k=i} w^{+}_k \;\Big|\; w_{1:N}\Big]\;=\;w_i \quad\text{for every } i
\label{eq:unbiased}
\end{equation}
(on average, the offspring of particle $i$ carry exactly its weight); \emph{exact count}, $N$ output particles; and \emph{weight conservation}, $\sum_{k} w^{+}_k=\sum_{i} w_i$.
\end{proposition}
Unbiasedness \eqref{eq:unbiased} implies that replacing the resampler does not
change what the weighted population estimates.

\begin{proposition}[Target preservation]
\label{prop:target}
Let $w_{1:N}$ be the input weights of particles $X_{1:N}$, let
$\mathcal F=\sigma(X_{1:N},w_{1:N})$ be the pre-resampling $\sigma$-field, and let
$(A_{1:N},w^{+}_{1:N})$ be Chopthin's output. For any bounded test function $\varphi$,
\begin{equation}
\mathbb E\Big[\sum_{k=1}^{N}w^{+}_k\,\varphi(X_{A_k})\;\Big|\;\mathcal F\Big]
=\sum_{i=1}^{N} w_i\,\varphi(X_i).
\label{eq:target}
\end{equation}
\end{proposition}

\emph{Proof.} All offspring of parent $i$ carry the same weight, say $v_i$, and there are $C_i$
of them. Grouping the output sum by parent gives
$\sum_k w^{+}_k\,\varphi(X_{A_k})=\sum_i C_i v_i\,\varphi(X_i)$; each $\varphi(X_i)$ is
$\mathcal F$-measurable. The resampler uses only the weights and external randomness, so
unbiasedness holds conditional on $\mathcal F$: $\mathbb E[C_i v_i\mid\mathcal F]=w_i$, which yields
\eqref{eq:target}. The same argument applies to any unbiased resampler.
\hfill$\square$

Thus Chopthin leaves the weighted empirical approximation of the intermediate targets unchanged
in conditional expectation; it changes the resampling variance and the genealogy, not the target
\citep{delmoral2004feynman,douc2005comparison}. A finite population remains an approximation of
$\pi_\alpha$, not an exact draw from it.

\subsection{Our choice of $\eta$}
\label{app:etachoice}
The floor of Proposition~\ref{prop:floor} has leading term $4 \eta N / (\eta + 1)^2$ for large
$N$. Setting this equal to a target fraction $\rho N$ of the population and solving for $\eta$
gives $\eta = (2 - \rho + 2\sqrt{1 - \rho})/\rho$. We target $\rho = 0.5$, which yields
$\eta = 3 + \sqrt{8} \approx 5.83$ and satisfies the validity condition $\eta \geq 4$. Setting
$\eta=4$ in the inversion gives $\rho=16/25$, so target fractions $\rho > 16/25$ would
require $\eta < 4$ and are not attainable under \eqref{eq:h}. The finite-$N$ floor sits a constant $2\sqrt2\approx2.83$ particles below the asymptote: at
$\eta = 3+\sqrt8$ the bound is exactly $N/2 - 2\sqrt2$, so at $N = 32$ it is
$\mathrm{ESS} \geq 13.17$ ($0.412N$), approaching $0.5N$ as $N$ grows.

\subsection{Solver and integration into Power-SMC}
\label{app:solver}

\paragraph{Models, benchmarks, and prompting.} We use the Hugging Face releases
\texttt{Qwen/Qwen2.5-Math-7B} and \texttt{Qwen/Qwen2.5-7B} (base variants) and
\texttt{Qwen/Qwen3-4B} (the post-trained release; the separate \texttt{-Base} checkpoint is not
used), in bfloat16; Qwen3-4B is prompted through its chat template with the thinking mode
disabled. The five benchmarks are MATH500 ($500$ problems) and GSM8K ($1319$) for mathematics,
AIME (competition mathematics; the $90$ problems of AIME 2022--2024), GPQA (graduate-level
science, four-way multiple choice; the Diamond split, $198$ problems), and HumanEval (code,
$164$ problems).

\paragraph{Decoding configuration.} The proposal is
$q(\cdot\mid x,y_{<t})\propto p_\theta(\cdot\mid x,y_{<t})^{1/\tau}$ with $\tau=1/\alpha=0.5$,
under nucleus truncation \citep{holtzman2020curious} at $\text{top-}p=0.9$. The exponent is
ramped from $1$ to $2$ over the first $100$ tokens, decoding runs for up to $4096$ new tokens,
and the ESS trigger $\kappa=0.5$ is evaluated at block boundaries every $B=64$ tokens. The
Power-SMC (systematic) and Chopthin arms are identical except for the resampler; Chopthin uses
$\eta=3+\sqrt8$. We use $\alpha=2$ rather than a larger exponent because at $N=32$ it sharpens
the target while leaving the weight dispersion that resampling acts on; all method components
are otherwise independent of $\alpha$.

We solve the threshold equation \eqref{eq:threshold} by bisection on the predicate
$\sum_i h^{\eta}_a(w_i)\ge N$: an enclosing bracket is
obtained by doubling and halving from $\max_i w_i$, and the bisection converges to a tolerance
of $10^{-12}$ to the supremum of the solution set, i.e., to the upper endpoint when the set is
an interval; the weights are processed in double precision on the host.
\citep{gandy2016chopthin} give an algorithm that
solves \eqref{eq:threshold} in expected linear time, but at our population sizes ($N \le 128$)
the simpler bisection suffices. Measured at $N=32$ across three weight regimes, the full
Chopthin call (root solve, thinning, and chopping) takes $0.19$--$0.26$ ms per event, against
roughly $10$--$15$ ms per batched token step and at least $B=64$ token steps between events; the
resampling step therefore costs about $2\%$ of one forward step and under $5\times10^{-4}$ of
the wall-clock time of the block it serves (systematic resampling is simpler still). Both resamplers expose the same
ancestor-index interface, so all downstream state management is shared verbatim: particle
sequences, termination flags, and the Transformer KV caches are reordered by $A_{1:N}$, and the
copy-on-write cache handling of \citep{azizi2026powersmc} applies to Chopthin without
modification.

Two further decoding details apply to all runs: the end-of-sequence token is masked in the
proposal for the first $100$ generated tokens, and the boxed-answer stopping rule is evaluated
at every step over a trailing window of $256$ tokens, so a particle may stop before the length
cap. Nucleus truncation, EOS masking, and the stopping rule restrict the sampler's
support, so it targets $\pi_\alpha$ over the reachable trajectories rather than
all sequences; both arms share these choices, so the comparison is unaffected. All runs use single NVIDIA RTX A6000 GPUs (48\,GB), PyTorch 2.5.1, and Transformers
4.45.0.

\paragraph{Paired randomness.} The token-sampling and resampling streams (the latter
covering the resampler's internal draws and the final weight draw) are initialized identically
per problem in both arms, so the two arms produce identical token streams until their first
divergent resampling event and comparisons are paired at the problem level.

\section{CCPS pseudocode}
\label{app:algorithm}

Algorithm~\ref{alg:cps} gives the complete sampler. It differs from Power-SMC in exactly two
places: the resampling step (lines~\ref{line:resample}--\ref{line:carry}), which invokes
Chopthin and carries its output weights, and the selection stage
(lines~\ref{line:dedup}--\ref{line:return}).

\begin{algorithm}[t]
\caption{Chopthin-Consensus Power Sampling (CCPS, ours).}
\label{alg:cps}
\small
\begin{algorithmic}[1]
\Require prompt $x$; model $p_\theta$; exponent $\alpha$; particles $N$; ratio bound $\eta$;
         trigger $\kappa$; block size $B$; ramp $T_{\mathrm{ramp}}$; max tokens $T_{\max}$;
         proposal $q(\cdot\mid x,y_{<t})\propto p_\theta(\cdot\mid x,y_{<t})^{1/\tau}$,
         $\tau=1/\alpha$, nucleus-truncated
\Ensure answer $\hat a_{\mathrm{maj}}$
\State $y^{(i)}\gets\emptyset$;\quad $U^{(i)}\gets 1$;\quad
       $\mathrm{done}^{(i)}\gets\mathrm{false}$\quad for $i=1,\dots,N$
\For{$t=1$ \textbf{to} $T_{\max}$}\Comment{decode $N$ particles as one batch}
  \State while the ramp is active, multiply each particle absorbed before step $t$ by the prefix
         factor $p_\theta(y^{(i)}\mid x)^{\alpha_t-\alpha_{t-1}}$
         \Comment{Appendix~\ref{app:powersmc}}
  \For{each active particle $i$}
    \State sample $y^{(i)}_t\sim q(\cdot\mid x,y^{(i)}_{<t})$;\quad
           $U^{(i)}\gets U^{(i)}\cdot\omega^{(i)}_t$
           \Comment{Eq.~\eqref{eq:incr}}
    \State \textbf{if} EOS or stopping rule \textbf{then}
           $\mathrm{done}^{(i)}\gets\mathrm{true}$
           \Comment{stops generating; stays in the population}
  \EndFor
  \If{$t\bmod B=0$ \textbf{and} $\mathrm{ESS}_t<\kappa N$}
      \Comment{Eq.~\eqref{eq:ess}}
    \State \label{line:resample}
           $(A_{1:N},w^{+}_{1:N})
           \gets\textsc{Chopthin}\big(W^{(1:N)}_t,\eta,N\big)$
           \Comment{Alg.~\ref{alg:chopthin}}
    \State reorder sequences, flags, and KV caches by $A_{1:N}$
    \State \label{line:carry}
           $U^{(k)}\gets w^{+}_k$ for all $k$
           \Comment{carry the weights; no uniform reset}
  \EndIf
\EndFor
\State $\Wfinal{i}\gets U^{(i)}/\sum_j U^{(j)}$ for all $i$
       \Comment{normalize final weights}
\State \label{line:dedup}
       $\mathcal T\gets$ deduplicate token-identical trajectories; pool weights $\bar W_u$
\State $\mathcal C\gets$ cluster answers by grader or behavioral equivalence
\State $C^\ast\gets
       \arg\max_{C\in\mathcal C}
       \lvert\{u\in\mathcal T:a_u\in C\}\rvert$
       \Comment{ties by
       $\sum_{u\in\mathcal T:\,a_u\in C}\bar W_u$;
       Sec.~\ref{sec:selection}}
\State \label{line:return}
       \Return $\hat a_{\mathrm{maj}}\gets\mathrm{rep}(C^\ast)$
\end{algorithmic}
\end{algorithm}
\section{Power-SMC incremental weights and the $\alpha$-ramp}
\label{app:powersmc}

This appendix restates the weight construction of \citep{azizi2026powersmc} in our notation;
nothing here is new. The incremental weight \eqref{eq:incr} expands, using
$\gamma_t(y_{1:t}\mid x)=p_\theta(y_{1:t}\mid x)^{\alpha_t}$, as
\begin{equation}
\omega_t
= p_\theta(y_{1:t-1}\mid x)^{\alpha_t-\alpha_{t-1}}\,
\frac{p_\theta(y_t\mid x,y_{<t})^{\alpha_t}}{q(y_t\mid x,y_{<t})}.
\label{eq:incr-expanded}
\end{equation}
The first factor reweights the whole prefix from exponent $\alpha_{t-1}$ to $\alpha_t$; it is
computable from the accumulated base log-probability of the prefix. In the absence of
resampling, the accumulated weight telescopes by induction to the
sequential-importance-sampling identity
\[
U^{(i)}_t
= \prod_{s\le t}\omega^{(i)}_s
= \frac{p_\theta(y^{(i)}_{1:t}\mid x)^{\alpha_t}}{\prod_{s\le t} q(y^{(i)}_s\mid x, y^{(i)}_{<s})}
= \frac{\gamma_t(y^{(i)}_{1:t}\mid x)}{\prod_{s\le t} q(y^{(i)}_s\mid x, y^{(i)}_{<s})},
\]
the exact importance weight for the intermediate target $\gamma_t$. With resampling, this
identity holds between consecutive resampling events; at the population level, the weighted
empirical approximation of the intermediate targets is preserved by the resampler's
unbiasedness (Proposition~\ref{prop:target}). Once the ramp completes ($\alpha_t=\alpha$), the
population targets $\pi_\alpha$ over the reachable sequences (exactly $\pi_\alpha$ only under a full-support proposal and EOS-only termination) in the stopping-policy sense of
Section~\ref{sec:background} \citep{delmoral2006sequential,azizi2026powersmc}, and the
incremental weight reduces to
$\omega_t = p_\theta(y_t\mid x,y_{<t})^{\alpha}/q(y_t\mid x,y_{<t})$.

Absorbed particles are handled by the same identity. A particle that stops at step $t_0$
contributes unit token-level increments for $t>t_0$, but the prefix factor
$p_\theta(y_{1:t_0}\mid x)^{\alpha_t-\alpha_{t-1}}$ continues to be applied while the ramp is
active, with a final catch-up to the full exponent at the end of decoding. An absorbed particle's trajectory therefore also carries weight $p_\theta(y_{1:t_0}\mid x)^{\alpha}$ divided by its
proposal mass, and stopping during the ramp does not distort the final target.

\section{Selector definitions}
\label{app:selector-defs}
\label{app:selector}

All selectors operate \emph{post hoc} on the same saved final population of $N$ weighted
particles; no selector invokes the model after the final populations and behavioral test
inputs have been generated. The behavioral test inputs used on HumanEval are synthesized
once per problem and per base model, before selection, and are shared verbatim by both
resampling arms and all selectors, so any accuracy difference between selectors is
attributable to selection alone.

Let $a^{(i)}$ be the parsed answer of particle $i$, $W^{(i)}_{\mathrm{final}}$ its
normalized final weight, and $g(a,a^{\star})\in\{0,1\}$ the task grader against the gold
answer $a^{\star}$. Power-SMC's default rule and two baselines are the \emph{weight draw}
(Section~\ref{sec:selection}), which returns $a^{(I)}$ with
$I\sim\mathrm{Categorical}(W^{(1:N)}_{\mathrm{final}})$; \textsc{argmax}, which returns
the answer of the single highest-weight particle; and \textsc{oracle},
$\max_i g(a^{(i)},a^{\star})$, the population ceiling used in the main text.

The semantic-majority selector \textsc{maj}, defined in Section~\ref{sec:selection},
aggregates over the distinct trajectories $\mathcal{T}$, each carrying a pooled weight
$\bar W_u=\sum_{i:\,y^{(i)}=u}W^{(i)}_{\mathrm{final}}$ used only to break ties; it is the
selector we report throughout. For the symbolic-answer benchmarks, $\mathrm{rep}(C)$ is
the first parsed answer that creates cluster $C$ under the greedy clustering procedure;
for HumanEval, it is the distinct program in $C$ with the largest pooled final weight.
Other aggregations of the same clusters are possible, for example weighting each vote by
its pooled weight or by an entropy or confidence score, or scoring with a process-reward
model, but we use plain majority.

A single equivalence check $e(a,b)$ underlies both scoring and clustering: scoring
compares a candidate to the gold answer, $g(a,a^{\star})=e(a,a^{\star})$, while
clustering compares candidates to one another and never consults the gold answer. Per
benchmark, $e$ combines, for MATH500, Hendrycks-style normalization
\citep{hendrycks2021math} with sympy simplification of the difference in the manner of
Minerva \citep{lewkowycz2022minerva} and the PRM800K grader
\citep{lightman2024verify}, with strict matching required when the reference argument
is an integer; absolute difference below $10^{-6}$ after comma stripping for GSM8K;
strict integer equality for AIME; and single-letter comparison after uppercasing for
GPQA. Answers are extracted by a three-tier ladder: a boxed expression; otherwise an
``answer is'' pattern in the final 300 characters; otherwise a tail fallback, which for
GPQA is a standalone capital letter A--D in the final 200 characters. Clustering is
greedy match-to-representative: each answer joins the first cluster whose representative
it grades equivalent to, with a string-equality shortcut, so no transitive closure is
computed. Particles whose answers fail to parse are excluded from the vote; if no
particle parses, the problem is scored incorrect for the majority selector. For
HumanEval, programs are clustered by exact behavioral signatures on a fixed ordered set
of model-generated inputs (Section~\ref{sec:selection}), adapting the execution-based
clustering of AlphaCode \citep{li2022alphacode} and the generated-test agreement of
CodeT \citep{chen2023codet}. The signature records the serialized output on each input;
when evaluation of an individual input raises an exception, the exception type is
recorded as part of the signature. Programs whose sandboxed execution fails or times out
before producing a signature are excluded. The selector falls back to the final-weight
draw if no generated inputs are available or if no program produces a usable behavioral
signature (this fallback exists only on HumanEval).

\paragraph{Table~\ref{tab:selector-ablation} details.} On HumanEval the ``majority''
selector is \emph{behavioral} majority (Section~\ref{sec:selection}); exact counts out
of $164$, per (Qwen2.5-Math-7B, Qwen2.5-7B, Qwen3-4B): systematic $+$ weight draw
$(96,120,117)$; systematic $+$ behavioral $(100,125,116)$; Chopthin $+$ weight draw
$(96,123,116)$; Chopthin $+$ behavioral $(101,126,116)$.

\section{Additional results}

\subsection{Oracle coverage values}
\label{app:coverage}

Table~\ref{tab:coverage} reports the absolute oracle coverage underlying the
differences of Figure~\ref{fig:gain}. Comparing each resampler's coverage against its majority-selector accuracy in
Table~\ref{tab:selector-ablation}, coverage exceeds accuracy in every cell
(median gap $4.2$ points; smallest gap $+1.5$, Qwen2.5-7B on AIME under
Chopthin), so selector headroom remains even after majority voting.

\begin{table}[h!]
\centering
\caption{Oracle coverage (\%): fraction of problems whose final population
contains at least one correct trajectory, per resampler. Bold marks the higher
value per column within each block; green rows mark the Chopthin arm (ours).
Coverage is measured over the correlated population of a single run and is a
ceiling for any selector, not i.i.d.\ pass@$N$.}
\label{tab:coverage}
\renewcommand{\arraystretch}{1.05}
\small
\begin{tabular}{llccccc}
\toprule
Model & Resampler & MATH500 & GSM8K & AIME & GPQA & HumanEval \\
\midrule
 & Systematic & 84.6 & 94.4 & 17.8 & 50.5 & 64.0 \\
\rowcolor{oursgreen}
\cellcolor{white}\multirow{-2}{*}{Qwen2.5-Math-7B}
 & Chopthin   & \textbf{87.0} & \textbf{94.8} & \textbf{22.2} & \textbf{54.5} & \textbf{68.3} \\
\midrule
 & Systematic & 81.2 & 94.1 & 13.7 & \textbf{50.5} & 83.5 \\
\rowcolor{oursgreen}
\cellcolor{white}\multirow{-2}{*}{Qwen2.5-7B}
 & Chopthin   & \textbf{83.0} & \textbf{94.8} & 13.7 & 49.0 & \textbf{84.8} \\
\midrule
 & Systematic & 83.8 & 93.7 & 17.8 & 39.9 & 72.6 \\
\rowcolor{oursgreen}
\cellcolor{white}\multirow{-2}{*}{Qwen3-4B}
 & Chopthin   & \textbf{84.8} & \textbf{93.9} & \textbf{20.0} & \textbf{45.5} & \textbf{73.8} \\
\bottomrule
\end{tabular}
\end{table}

\subsection{Reported GRPO reference values}
\label{app:grpo}

Table~\ref{tab:grpo} situates CCPS against reported GRPO post-training results under
unmatched evaluation protocols. Training-free CCPS is competitive on MATH500 and stronger
on HumanEval, while GRPO retains a clear advantage on GPQA on both models.

\begin{table}[h!]
\centering
\caption{Reported post-training reference values for GRPO \citep{shao2024deepseekmath}, as
reported by \citep{karan2026powersampling} and tabulated by \citep{nguyen2026apps}, alongside
CCPS. The evaluation protocol (prompts, checkpoints, graders, budgets) is not matched to
ours; these numbers are context, not a controlled comparison.}
\label{tab:grpo}
\renewcommand{\arraystretch}{0.9}
\footnotesize
\begin{tabular}{llccc}
\toprule
Model & Method & MATH500 & GPQA & HumanEval \\
\midrule
Qwen2.5-Math-7B & GRPO (post-training, reported) & 78.5 & \textbf{39.9} & 53.7 \\
\rowcolor{oursgreen}
                & CCPS (training-free, ours) & \textbf{81.4} & 33.8 & \textbf{61.6} \\
\midrule
Qwen2.5-7B      & GRPO (post-training, reported) & 74.0 & \textbf{35.4} & 56.1 \\
\rowcolor{oursgreen}
                & CCPS (training-free, ours)     & \textbf{76.0} & 30.3 & \textbf{76.8} \\
\bottomrule
\end{tabular}
\\[3pt]
{\footnotesize Bold marks the higher value per model and benchmark; green rows are our method.}
\end{table}
\end{document}